\def\year{2021}\relax
\documentclass[letterpaper]{article} 
\usepackage{aaai21}  
\usepackage{times}  
\usepackage{helvet} 
\usepackage{courier}  
\usepackage[hyphens]{url}  
\usepackage{graphicx} 
\usepackage{multirow}
\usepackage{booktabs}
\usepackage{amsfonts}
\usepackage{graphicx}
\usepackage{amsmath}
\usepackage{subcaption}
\usepackage{arydshln}

\usepackage{stfloats}
\usepackage{natbib}  
\usepackage{caption} 
\usepackage[misc]{ifsym}

\def\aaai{0}

\newtheorem{theorem}{Theorem}{\itshape}{\rmfamily}
{\itshape}{\rmfamily}
\newtheorem{lemma}{Lemma}{\itshape}{\rmfamily}
{\itshape}{\rmfamily}
{\itshape}{\rmfamily}

{\itshape}{\rmfamily}

\newcommand{\sq}{\hbox{\rlap{$\sqcap$}$\sqcup$}}
\newcommand{\qed}{\hspace*{\fill}\sq}
\newenvironment{proof}{\noindent {\bf Proof.}\ }{\qed\par\vskip 4mm\par}
\newenvironment{proofof}[1]{\bigskip \noindent {\bf Proof of #1:}\quad }
{\qed\par\vskip 4mm\par}

\newcommand{\eq}[2]{\begin{equation}\label{#1} #2 \end{equation}}
\newcommand{\eqm}[2]{\begin{equation}\label{#1}\begin{split} #2 \end{split}\end{equation}}

\newcommand{\eqmx}[1]{\begin{equation*} \begin{split} #1 \end{split}\end{equation*}}

\newcommand{\E}[1]{\operatorname*{\mathbb{E}}_{\substack{#1}}}

\newcommand{\define}[0]{\doteq}
\newcommand{\indicator}[1]{\mathbb{I}(#1)}

\newcommand{\F}{\mathcal{F}}
\newcommand{\Fbar}{\bar{\mathrm{F}}}
\newcommand{\G}{\mathrm{G}}
\newcommand{\M}{\mathcal{M}}

\newcommand{\Yhat}[2]{\hat{Y}^{\scriptscriptstyle (#1)}_{#2}}
\renewcommand{\i}[0]{^{\scriptscriptstyle (i)}}

\title{Simpson's Bias in NLP Training}

\if\aaai1
 \author {
     {Fei Yuan}\textsuperscript{\rm 1} 
     {\quad  Longtu Zhang}\textsuperscript{\rm 2} 
     {\quad  Huang Bojun}\textsuperscript{\rm 2}\thanks{
     	Correspondence to:~ \texttt{bojhuang@gmail.com} \vspace{0.1in}
 	 }
     {\quad Yaobo Liang}\textsuperscript{\rm 3} \\
 }

 \affiliations{
     \textsuperscript{\rm 1} {University of Electronic Science and Technology of China} \\
     \textsuperscript{\rm 2} {Rakuten Institute of Technology, Rakuten, Inc.} \\
     \textsuperscript{\rm 3} {Microsoft Research Asia} \\
 	{feiyuan@std.uestc.edu.cn, longtu.zhang@rakuten.com, bojun.huang@rakuten.com, yalia@microsoft.com}  
 }
\else
\author {
	\quad\quad\quad Fei Yuan \quad\quad\quad\quad\quad~
	Longtu Zhang 			 \quad\quad\quad\quad\quad~
	Huang Bojun \thanks{
	Correspondence to :~ \texttt{bojhuang@gmail.com} \vspace{0.1in} 
	} 						 \quad\quad\quad\quad~
	Yaobo Liang 			 ~~
	\\
}

\affiliations{
	\small
	{\quad\quad University of Electronic \quad\quad\quad}
	{Rakuten Institute of Technology 	 \quad\quad}
	{Rakuten Institute of Technology 	 \quad\quad\quad~~}
	{Microsoft Research} 
	\\
	
	{Science and Technology of China  \quad\quad\quad\quad~~}
	{Rakuten, Inc. \quad\quad\quad\quad\quad\quad\quad\quad\quad~~}
	{Rakuten, Inc. \quad\quad\quad\quad\quad\quad\quad\quad\quad~~}
	{Microsoft \quad\quad} 
	\\
	
}
\fi

\begin{document}

 \maketitle

\begin{abstract}
In most machine learning tasks, we evaluate a model $\M$ on a given data population $S$ by measuring a population-level metric $\F(S;\M)$. Examples of such evaluation metric $\F$ include precision/recall for (binary) recognition, the F1 score for multi-class classification, and the BLEU metric for language generation. On the other hand, the model $\M$ is trained by optimizing a sample-level loss $\G(S_t;\M)$ at each learning step $t$, where $S_t$ is a subset of $S$ (a.k.a. the mini-batch). Popular choices of $\G$ include cross-entropy loss, the Dice loss, and sentence-level BLEU scores. A fundamental assumption behind this paradigm is that the mean value of the sample-level loss $\G$, if averaged over all possible samples, should \emph{effectively represent} the population-level metric $\F$ of the task, such as, that $\E{}[ \G(S_t;\M) ] \approx \F(S;\M)$.  

In this paper, we systematically investigate the above assumption in several NLP tasks. We show, both theoretically and experimentally, that some popular designs of the sample-level loss $\G$ may be inconsistent with the true population-level metric $\F$ of the task, so that models trained to optimize the former can be substantially sub-optimal to the latter, a phenomenon we call it,  \emph{Simpson's bias}, due to its deep connections with the classic paradox known as \emph{Simpson's reversal paradox} in statistics and social sciences. 
\end{abstract}

\section{Introduction}
\label{sec_intro}


Consider the following standard and general paradigm of NLP training: given a corpus $S$ consisting of $n$ samples, each indexed by $i = \{1,\dots,n\}$, the training of NLP model $\M$ aims at optimizing a corpus-level objective $\F(S;\M)$. 
For example, a popular training method follows the maximum likelihood estimation (MLE) principle, in which a sample is a $(x_i,y_i)$ pair with $x_i$ being a decision context, which is usually one or more sentences in NLP tasks, and $y_i$ being a desired atomic decision, which is usually a token in generative tasks or a class label in discriminative tasks. The corpus-level objective $\F$ that MLE-oriented training aims at maximizing is the log-likelihood of the whole corpus: $\F_{\texttt{MLE}}(S;\M) \define \sum_{i=1}^n \log \M(x_i,y_i)$.

The MLE objective is relatively easy to optimize because we can construct a sample-level loss function $\G(i;\M)$ such that the sample average $\Fbar(S;\M) \define \frac{1}{n}\sum_{i=1}^n \G(i;\M)$ can ``effectively represent'' $\F_{\texttt{MLE}}(S;\M)$ as a surrogate objective of the optimization. Specifically, since $\F_{\texttt{MLE}}$ itself is \emph{additive} with respect to the samples in $S$, we can simply take the CE loss $\G_{\texttt{MLE}}(i;\M) \define \F_{\texttt{MLE}}(\{i\};\M)$, which gives

\vspace{-0.1in} \small 
\eqmx{  
	\Fbar_{\texttt{MLE}}(S;\M) 
	&= \frac{1}{n}\sum_{i=1}^n \F_{\texttt{MLE}}(\{i\};\M) 
	= \frac{1}{n}\sum_{i=1}^n \log \M(x_i,y_i) 
	\\& \propto \F_{\texttt{MLE}}(S;\M)
.}
\normalsize \noindent
The average form of $\Fbar_{\texttt{MLE}}$ admits efficient stochastic-gradient optimization (which requires the objective to be a population mean such that its gradient can be unbiasedly estimated by the gradient of the sample mean over a random mini-batch), and the proportionality between $\Fbar_{\texttt{MLE}}$ and $\F_{\texttt{MLE}}$ guarantees that an optimal (better) solution of the former is also an optimal (better) solution of the latter.

However, it is rare that a task directly uses $\F_{\texttt{MLE}}$ as the end-to-end \emph{evaluation metric}. Instead, common evaluation metrics used in practice include accuracy, precision/recall/F1 (for discriminative tasks), and BLEU~\cite{papineni2002bleu} (for machine translation and other language generation tasks). While a model trained with $\G_{\texttt{MLE}}$ may well optimize the corresponding MLE objective $\F_{\texttt{MLE}}$, it does not necessarily optimize the true evaluation metric of the task. For this reason, researchers have proposed to optimize alternative objective $\F$ that is closer to, or in some cases equal to, the true evaluation metric used at testing time. For example, the \emph{Dice loss}~\cite{DBLP:conf/acl/LiSMLWL20} has been recently proposed for tasks such as Paraphrase Similarity Matching (PSM) and Named Entity Recognition (NER) because of its similarity to the F1 metric used in these tasks. Similarly, \emph{sentence-level BLEU} scores have been used in sentence-level training for machine translation due to its correspondence to the true corpus-level BLEU metric~\cite{2016:ranzato,wu2016google,2018:edunov}.

Unfortunately, these alternative learning objectives posed new challenges in optimization. Specifically, metrics like F1 and BLEU (and many others) are not \emph{sample-separable}, meaning that they cannot be converted proportionally or monotonically into an averaged form $\Fbar$ as in the case of MLE. Consequently, while the \emph{intended} objectives $\F_{\texttt{F1}}$ and $\F_{\texttt{BLEU}}$ are more aligned with the evaluation metric of the corresponding tasks, what the training algorithms are truly optimizing is usually the \emph{averaged-form} objectives $\Fbar_{\texttt{F1}}$ and $\Fbar_{\texttt{BLEU}}$, and models thus trained could improve the averaged objective $\Fbar$ while at the same time being worse with respect to the intended objective $\F$.

In this paper, we call the disparity mentioned above, \emph{Simpson's bias}. It is a bias between non-separably aggregated objective $\F$ and its corresponding averaged form $\Fbar$. The name is inspired by the classic paradox known as \emph{Simpson's reversal} in statistics and social sciences, which refers to a class of conflicting conclusions obtained when comparing two ``candidates'' based on their aggregated performance and based on their per-case performance. In the following, we will give a systematic analysis on how a similar effect can widely arise in the context of machine learning when designing sample-level loss for many popular metrics including precision, recall, Dice Similarity Coefficient (DSC), Macro F1, and BLEU. We then experimentally examine and verify the practical impacts of the Simpson's bias on the training of state-of-the-art models in three different NLP tasks: Paraphrase Similarity Matching (with the DSC metric), Named Entity Recognition (with the Macro-F1 metric), and Machine Translation (with the BLEU metric).

\section{The Simpson's Bias}

As discussed in the last section, the ultimate goal of NLP training is to optimize a set function $\F(S;\M)$ which is a corpus-wise \emph{aggregated} measurement of model $\M$'s performance on given data set $S=\{1 \dots n\}$. On the other hand, the model $\M$ is typically trained by following the gradient direction of a sample-level loss $\G(i;\M )$ on random sample $i\in S$.
\footnote{
	When mini-batch is used, the algorithm generates a random batch $S_t\subset S$ at each optimization step $t$ and follows the gradient direction of batch-wise averaged loss $\frac{1}{|S_t|}\sum_{i\in S_t} \G(i;\M)$.
} 
Such training is expected to find an extreme point of the \emph{averaged} performance $\Fbar_G(S;\M) \define \frac{1}{n} \sum_{i\in S} \G(i;\M)$. 

We will pay special attention to the ``naive'' sample-level loss $\G_\F(i;\M)\define \F(\{i\};\M)$, which uses the same metric $\F$ to measure a single sample. We use the $\Fbar$ without subscript to denote the corpus-wise averaged performance corresponding to this particular sample loss $G_{\F}$, so $\bar{F}\define \frac{1}{n} \sum_{i\in S} \G_{\F}(i;\M)$. Note that every well-defined set function $\F$ is conjugated with such an $\Fbar$, which is the arithmetic average of $\F$ over all singletons of $S$. On the other hand, the function $\F$ itself, when used as a performance metric in machine learning, often involves some form of ``complex averaging'' over $S$ as well. We are interested to understand whether, or to what extent, a model optimized for the arithmetic average $\Fbar$ can also perform well w.r.t. the ``complex'' average $\F$, for various specific forms of $\F\neq \F_{\texttt{MLE}}$. 




\subsection{Special case 1: Ratio of Sums (RoS)} 
This is a very common family of metric $\F$, which computes the ratio of two summations over the set $S$. Let $A_i$ and $B_i$ be two quantities defined on each sample $i$, the RoS family of $\F$ is generally in the form of
\eq{MoR}{
	\F(S) = \frac{\sum_{i=1}^n A_i}{\sum_{i=1}^n B_i}
} 
and the corresponding ``naively''-averaged metric is 
\eq{}{
	\Fbar(S) = \frac{1}{n} \sum_{i=1}^n \F(\{i\}) = \frac{1}{n} \sum_{i=1}^n \frac{A_i}{B_i}
.}
In the above, we have omitted $\M$, which is considered given in this section. As a best case, $\F$ of the RoS family equals $\Fbar$ in the following two conditions:

\vspace{0.1in}
\textbf{Type-1:} If $B_i \equiv B$ for some constant $B$, then 
\begin{equation*}
\Fbar(S) 
= \frac{1}{n} \sum_{i} \frac{A_i}{B_i} 
=  \frac{1}{nB}\sum_{i} A_i 
= \frac{\sum_i A_i}{\sum_i B_i} 
= \F(S) 
\end{equation*}

\textbf{Type-2:} If $\frac{A_i}{B_i} \equiv r$ for some constant $r$, then
\begin{equation*}
\Fbar(S) = \frac{1}{n}\sum_i \frac{A_i}{B_i} = r = \frac{\sum_i r B_i}{\sum_i B_i} = \frac{\sum_i A_i}{\sum_i B_i} = \F(S)
\end{equation*}

Depending on precise definitions of $A_i$ and $B_i$, the RoS family subsumes many concrete metrics used in NLP tasks. We discuss three popular RoS metrics in the following.

\paragraph{Scenario 1.a: Accuracy} 
Let $y_i$ be a ground-truth decision on sample $i$ and $\hat{y}_i$ the decision output by the model $\M$, the \emph{accuracy} of $\M$ on data set $S$ of size $n$ is
\eq{}{
	\F_{\texttt{AC}} = \frac{\sum_{i=1}^n \indicator{y_i=\hat{y}_i}}{n}
}
which is a special case of $\eqref{MoR}$ with $A_i=\indicator{y_i=\hat{y}_i}$ and $B_i=1$, where $\indicator{\cdot}$ is the indicator function.

Accuracy is the simplest case in our analysis, which does not suffer from the Simpson's bias at all, as it satisfies the type-1 condition above. In other words, optimization based on the naive sample-level loss $\G_{\texttt{AC}}(i;\M) =\indicator{y_i = \hat{y}_i}$ will maximize exactly the accuracy $\F_{\texttt{AC}}=\Fbar_{\texttt{AC}}$.

Note that in supervised learning, the sample loss $\G$ may further need to be differentiable, in which case the indicator variable $\indicator{y_i = \hat{y}_i}$ is usually approximated in practice. For example in \emph{binary recognition} problems, which ask to judge if each sample $i$ is positive or negative (w.r.t. some feature of interest), the model $\M$ is usually set to output a probability $p_i=\M(x_i)$, and differentiable sample losses such as $(p_i-y_i)^2$ are used, essentially as smoothed variants of the discrete loss $\indicator{y_i \neq \hat{y}_i}=1-\indicator{y_i = \hat{y}_i}$. 

We do not consider errors from such differentiablization tricks as part of the Simpson's bias under discussion, as the former is mostly a limit of only specific (types of) learning algorithms. In contrast, the Simpson's bias that we are studying in this paper is concerned more with \emph{intrinsic} properties of the learning objectives themselves. For example, the exact sample-level accuracy $\G_{\texttt{AC}}(i;\M)=\indicator{y_i = \hat{y}_i}$ can indeed be directly optimized through \emph{reinforcement learning} algorithms, in which case the learning algorithm is equivalently optimizing exactly the corpus-wise accuracy $\F{_\texttt{AC}}$.  

\paragraph{Scenario 1.b: Precision/Recall} 
~~While being applicable to almost all discrete decision tasks, accuracy can be problematic for tasks with imbalanced data. For example, in binary recognition problems, a model always outputting negative would have very high accuracy if positive samples are rare. Precision and recall are standard evaluation metrics used in binary recognition tasks to solve this problem.

In binary recognition problems, let $y_i \in \{0,1\}$ be the true label of sample $i$, $y_i=0$ for negative sample and $y_i=1$ for positive sample. Let $\hat{y}_i \in \{0,1\}$ be the predicted label by model $\M$, $\hat{y}_i=0$ for negative output and $\hat{y}_i=1$ for positive output. The \emph{precision} on a data set $S$ of size $n$ is
\eq{precision}{
	\F_{\texttt{P}} = \frac{ \sum_{i=1}^n y_i \hat{y}_i }{ \sum_{i=1}^n \hat{y}_i }.
}    

It is clear that $\F_\texttt{P}$ can be seen as a RoS metric with $A_i=y_i\hat{y}_i$ and $B_i=\hat{y}_i$. But strictly speaking, $\F_{\texttt{P}}$ is not a completely well-defined metric as its denominator $\sum_i \hat{y}_i$ can be zero. This issue becomes more evident when we try to write its naively-conjugated form $\Fbar_\texttt{P} = \frac{1}{n}\sum_i \frac{y_i\hat{y}_i}{\hat{y}_i}$. For this reason, we turn to consider the \emph{smoothed precision}
\eq{precision_gamma}{
	\F_{\texttt{P}^\gamma} = \frac{ 
		\gamma + \sum_{i=1}^n y_i \hat{y}_i 
	}{ 
		\gamma + \sum_{i=1}^n \hat{y}_i 
	}
}
which is a genuine RoS metric that subsumes the vanilla precision $\F_{\texttt{P}}$ with $\gamma = 0$, and its average form
\eq{}{
	\Fbar_{\texttt{P}^\gamma} = \frac{1}{n}\sum_i \F_{\texttt{P}^\gamma}(i) = \frac{1}{n}\sum_i \frac{\gamma + y_i\hat{y}_i}{\gamma + \hat{y}_i}
} 
is always well defined for $\gamma \neq 0,-1$. 

Unlike accuracy, the (smoothed) precision metrics do not satisfy either of the two equality conditions above, and may suffer from the Simpson's bias \emph{in general}. This is especially true for $\gamma \in [0,1]$ which is the commonly used smoothing constant in existing practice, as Section \ref{sec_exp} will later demonstrate. However, the following theorem shows that the Simpson's bias for smoothed precision may disappear under a \emph{special} (and \emph{unusual}) smoothing term $\gamma^* < 0$, such that the smoothed precision $\F_{P^{\gamma}}$ equals precisely to its conjugate metric $\Fbar_{P^{\gamma}}$ under this special $\gamma^*$.

\begin{theorem} \label{thm_precision_1}
	$\Fbar_{\texttt{P}^\gamma} = \F_{\texttt{P}^\gamma}$ if $\gamma = - \frac{n-\sum_i \hat{y}_i}{n-1}$ and $\sum_i \hat{y}_i \geq 2$. 
\end{theorem}

\noindent
More importantly, there turns out to be also a special smoothing term $\gamma^{\texttt{P}} < 0$, such that the averaged sample-level precision smoothed by this particular $\gamma^{\texttt{P}}$ happens to equal precisely the \emph{original} precision metric $\F_\texttt{P}$. 

\begin{theorem} \label{thm_precision_2}
	$\Fbar_{\texttt{P}^\gamma} = \F_{\texttt{P}^0}$ if $\gamma = \frac{\sum_i \hat{y}_i}{n} - 1$.
\end{theorem}

\noindent
\if\aaai1
See the proofs in appendix of our full paper.
\else
See the proofs of Theorem \ref{thm_precision_1} and Theorem \ref{thm_precision_2} in Appendix \ref{sec_proof}.
\fi

According to Theorem \ref{thm_precision_2}, the special smoothing term $\gamma^{\texttt{P}}$ is the negated negative-output-rate of the model $\M$. The theorem says that although the original precision metric does suffer from the Simpson's bias (in the sense that $\bar{F}_{\texttt{P}^0} \neq \F_{\texttt{P}^0}$), the bias can be completely resolved by using the special smoothing term $\gamma^{\texttt{P}}$. Note that $\gamma^{\texttt{P}}$, as a \emph{negative} smoothing term, is outside the typical value range of smoothing-term tuning in previous works (which usually used $\gamma\in [0,1]$).~\footnote{We also remark that the smoothing term was previously only used to make the precision metric well defined on singleton samples, not for solving the Simpson's bias.}  

\vspace{0.1in}
Finally, the \emph{recall} metric is symmetrically defined as $\F_\texttt{R} = \frac{\sum_i y_i\hat{y}_i}{\sum_i y_i}$, thus all the observations about precision as discussed also symmetrically apply to recall. In particular, we have $\Fbar_{\texttt{R}^\gamma}=\F_{\texttt{R}}$ for $\gamma=\gamma^\texttt{R} = \sum_i y_i / n -1$.

\paragraph{Scenario 1.c: Dice Coefficient} 
Dice similarity coefficient (DSC) is a measure to gauge the similarity of two overlapped (sub-)sets. In binary recognition problems, DSC is used as a performance metric that combines precision and recall.

Specifically, the DSC metric is the harmonic mean of precision and recall. Following the same formulation with Scenario 1.b, we can write
\eq{dsc}{
	\F_{\texttt{DSC}}(S) = \frac{2 \cdot \F_\texttt{P}(S) \cdot \F_\texttt{R}(S)}{\F_\texttt{P}(S) + \F_\texttt{R}(S)}
	= \frac{ \sum_{i=1}^n 2 y_i\hat{y}_i }{ \sum_{i=1}^n (y_i+\hat{y}_i) }
,} 
which is a RoS metric with $A_i = 2y_i\hat{y}_i$ and $B_i=y_i + \hat{y}_i$. We can also similarly generalize DSC to smoothed variant
\eq{dsc_gamma}{
	\F_{\texttt{DSC}^\gamma}(S) = \frac{ 
		\gamma + \sum_{i=1}^n 2 y_i\hat{y}_i  
	}{ 
		\gamma + \sum_{i=1}^n (y_i+\hat{y}_i) 
	}
,}
which has conjugated average-form
\eq{dsc_fbar}{
\Fbar_{\texttt{DSC}^\gamma} = \frac{1}{n}\sum_i \G_{\texttt{DSC}^\gamma}(i) = \frac{1}{n} \sum_i \frac{\gamma + 2y_i\hat{y}_i}{\gamma + y_i + \hat{y}_i}
}

The following theorem shows an interesting connection between DSC and accuracy. 
\if\aaai1
See the proofs in appendix of our full paper.
\else
See the proofs in Appendix \ref{sec_proof}.
\fi

\begin{theorem} \label{thm_dsc}
	$\Fbar_{\texttt{DSC}^\gamma}(S) = 1 - \frac{ |\{y_i\neq \hat{y}_i\}| }{ (1+\gamma)n }$ for $\gamma\neq 0,-1,-2$.
\end{theorem}

When $\gamma \approx 0$, the right-hand side of Theorem \ref{thm_dsc} is very close to the value of accuracy. So, it turns out that averaging the \emph{nearly un-smoothed} sample-level DSC gives us the corpus-level accuracy: $\Fbar_{\texttt{DSC}^\gamma} \approx \F_{\texttt{AC}}$ for $\gamma\approx 0$. In other words, Theorem \ref{thm_dsc} implies that the original DSC metric $\F_\texttt{DSC}$ (which is approximately $\F_{\texttt{DSC}^\gamma}$ with $\gamma \approx 0$, see \eqref{dsc_gamma}) does not only have the Simpson's bias, but the bias in this metric is so significant that its average-form conjugate $\Fbar_{\texttt{DSC}^\gamma}$ with $\gamma \approx 0$ has been completely distorted towards another metric (i.e. towards accuracy $\F_{\texttt{AC}}$). 
 
\if\aaai1 
\else
Moreover, Theorem \ref{thm_dsc} further implies that the Simpson's bias in DSC cannot be resolved by any smoothing term $\gamma$. Specifically, the theorem asserts that the smoothed averaged DSC $\Fbar_{\texttt{DSC}^\gamma}$ is monotonic to the error rate $\frac{ |\{y_i\neq \hat{y}_i\}| }{n}$ under \emph{any} admissible $\gamma$, which thus is monotonic to correction rate (i.e., accuracy) as well. This means optimizing the average-form DSC under whatever admissible smoothing term $\gamma$ will be equivalent to optimizing just the accuracy. In other words, in any binary recognition problem where the DSC metric is preferred over accuracy, the (potential) advantage of direct DSC optimization would be \emph{completely} offset by the Simpson's bias, no matter how we tune the smoothing constant.  
\fi

\subsection{Special case 2: Macro-F1} 

The DSC metric can be further extended to \emph{multi-class classification} problems, in which the model $\M$ is asked to classify each sample input $x_i$ into one of $K$ predefined classes. The ground-truth label $y_i\in \{0,1\}^K$ is a categorical variable whose $k$-th component $y_{i,k}$ is $1$ if sample $i$ is from class $k$, otherwise $y_{i,k}=0$. The decision of the model is similarly encoded by a one-hot vector $\hat{y}_i = \texttt{hardmax}(p_i) \in \{0,1\}^K$, where $p_i=\M(x_i)\in [0,1]^K$ is the model output under $x_i$. 

For given class $k$, the decision of the model is making binary recognition on the particular class $k$, thus all the metrics discussed so far applies in a per-class sense. Specifically, the model's \emph{precision for class $k$} is $P_k(S) = \frac{\sum_i y_{i,k} \cdot \hat{y}_{i,k}}{\sum_i \hat{y}_{i,k}}$, and its \emph{recall for class $k$} is $R_k(S) = \frac{\sum_i y_{i,k} \cdot \hat{y}_{i,k}}{\sum_i y_{i,k}}$. The \emph{DSC for class $k$} is, accordingly, $DSC_k(S) = \frac{\sum_i 2 \cdot y_{i,k} \cdot \hat{y}_{i,k}}{\sum_i y_{i,k} + \hat{y}_{i,k}}$. The $\emph{F1 score}$ of the model is the mean DSC value averaged over all classes,~\footnote{
	\eqref{f1} is usually called \emph{Macro-F1}, although the same name was also used for a similar but different metric~\cite{opitz2019macro}. Other F1 variants also exist, such as Micro-F1. \eqref{f1} is the evaluation metric used in tasks that we will experimentally examine later. 
} 
denoted as
\eq{f1}{
	\F_{1}(S) = \frac{1}{K} \sum_{k=1}^K DSC_k(S) = \sum_k \frac{\sum_i 2 \cdot y_{i,k} \cdot \hat{y}_{i,k}}{\sum_i y_{i,k} + \hat{y}_{i,k}} / K
} 

The F1 metric is a linear sum of several RoS metrics, but itself is not a RoS metric. The corresponding (smoothed) average-form F1 is
\eq{fbar_1}{
	\Fbar_{1}^\gamma(S) = \frac{1}{n} \sum_{i=1}^n \F_{1}^\gamma(\{i\}) = 
	\sum_i \sum_k \frac{
		\gamma + 2 \cdot y_{i,k} \cdot \hat{y}_{i,k}
	}{
		\gamma + y_{i,k} + \hat{y}_{i,k}
	} / Kn
.}
From Theorem \ref{thm_dsc} we know that the average-form F1 (that is, $\Fbar_{1}^\gamma$ with $\gamma\approx 0$) is equivalent to an ``mean-accuracy-over-class'' metric, which is different from the aggregated F1 metric (and is different from the multi-class accuracy metric actually used in multi-classification tasks too). 

Despite the Simpson's bias in F1 as discussed, the average-form F1\eqref{fbar_1} has inspired \citet{milletari2016v} to introduce the \emph{Dice Loss}, defined as 
\eq{dice}{
\Fbar_{\texttt{DL}}(S) =  \frac{1}{n}\sum_i \G_{\texttt{DL}}(i) = \sum_i \sum_k \frac{
	\gamma + 2 \cdot y_{i,k} \cdot p_{i,k}
}{
	\gamma + y_{i,k}^2 + p_{i,k}^2
} / Kn	
.}

Besides the differentiablization trick, the Dice loss \eqref{dice} further uses the squared terms $y_{i,k}^2$ and $p_{i,k}^2$ in denominator for faster training. \citet{DBLP:conf/acl/LiSMLWL20} has proposed to adopte the Dice loss to train models in a number of NLP tasks.

\subsection{Special case 3: BLEU} \label{sec:bleu} 

BLEU is a widely used evaluation metric in machine translation(MT) and question answering (QA). Given a parallel corpus $S$ consisting of $n$ sentence pairs $(X\i, Y\i)$, $X\i$ being the source sentence and $Y\i$ a reference translation, the MT model $\M$ will generate a translation $\Yhat{i}{}$ for each $i\in \{1 \dots n\}$. The BLEU score of the model $\M$ on such a data set $S$ is defined as $\texttt{BLEU}(S;\M) =$

\vspace{-0.1in} \small
\eqmx{
\mathbf{GM}_{k=1}^4 \left( \frac{\sum_i H_k\i}{\sum_i L_k\i} \right)
	\cdot \min \left( \exp \Big( 1-\frac{\sum_i M_1\i}{\sum_i L_1\i} \Big) ~,~ 1 \right)
}
\normalsize \noindent
where $L_k\i$ is the total number of n-grams of length $k$ in $\Yhat{i}{}$, $H_k\i$ is the number of ``matched'' n-grams of length $k$ in $\Yhat{i}{}$, $M_1\i$ is the total number of 1-grams in $Y\i$, and $\mathbf{GM}_{k=1}^4$ means taking the geometric mean over $k=1,2,3,4$.

To subsume the BLEU metric into our framework, define

\vspace{-0.1in}
\small
\eqm{f_bleu}{
	&\F_{\texttt{BLEU}}(S;\M) 
	\\=~& \log \texttt{BLEU}(S;\M) - 1 
	\\=~& \frac{1}{4}\log \Big( \frac{\sum_i H_1\i}{\sum_i L_1\i} \Big) + 
	\frac{1}{4}\log \Big( \frac{\sum_i H_2\i}{\sum_i L_2\i} \Big) +
	\frac{1}{4}\log \Big( \frac{\sum_i H_3\i}{\sum_i L_3\i} \Big) \\&+ 
	\frac{1}{4}\log \Big( \frac{\sum_i H_4\i}{\sum_i L_4\i} \Big) -
	\max\Big( \frac{\sum_i M_1\i}{\sum_i L_1\i}, 1 \Big)
}
\normalsize \noindent
which is equivalent to the exact BLEU metric in terms of model training. Similar to $\F_1$, the $\F_{\texttt{BLEU}}$ metric is also an aggregation of five RoS sub-metrics. However, different from $\F_1$, the RoS sub-metrics in $\F_{\texttt{BLEU}}$ will each go through a nonlinear transformation before summing over together. 
The corresponding average-form BLEU is

\vspace{-0.1in} \small
\eqm{}{
	\Fbar_{\texttt{BLEU}}(S) 
	&= \frac{1}{n} \sum_i \G_{\texttt{BLEU}}(i) = \frac{1}{n} \sum_i \F_{\texttt{BLEU}}(\{i\}) \\
	&= 
	\frac{1}{n} \sum_i \Big(~ 
		-\max\big(1, \frac{M_1\i}{L_1\i}\big) + \sum_{k=1\dots4} \frac{1}{4} \log \frac{H_k\i}{L_k\i}
	~~\Big)  
.}
\normalsize \noindent
Note that in $\Fbar_{\texttt{BLEU}}$, a sample is a sentence, and the metric computes a \emph{sentence-level BLEU} score~\cite{chen2014systematic} for each sentence $i$, then takes the arithmetic mean over all sentence-level scores. Sentence-level training could be conducted based on$\Fbar_{\texttt{BLEU}}$, as have been explored by many authors~\cite{2016:ranzato,DBLP:conf/acl/ShenCHHWSL16,wu2016google,2017:bahdanau,2018:wu,2018:edunov}, \emph{if} the sentence-averaged BLEU indeed serves as a good proxy to the true evaluation metric $\F_{\texttt{BLEU}}$, a presumption that we will experimentally examine in later sections.

\section{Connections to Simpson's Paradox}
Our naming of the bias between corpus-level metric $\F$ and its average-form conjugate $\Fbar$ is largely inspired by its connection with the famous notion, \emph{Simpson's reversal paradox}, which we will explain in this section. 

\emph{Simpson's reversal} often refers to the statistical observation that a candidate method/model is better in each and every case, but is worse in terms of the overall performance. For example, let $\M_1$ be a new medical treatment that is better than the baseline method $\M_0$ in terms of survival rate $\F$ for both the group of male patients and the group of female patients, it turns out that $\M_1$ could have a lower survival rate than $\M_0$ for the combined group of all patients, as famously shown by~\citet{blyth1972simpson}. 


Many people feel surprising, or even paradoxical, when they observe the Simpson's reversal. \citet{blyth1972simpson} was the first to call this phenomenon, \emph{Simpson's paradox}, named after Edward H. Simpson for his technical notes~\cite{simpson1951interpretation} that proposed to study the phenomenon more carefully. On the other hand, Simpson's reversal, as a mathematical fact, is not too rare in real-world experiences. \citet{2009:likely} show that the reversal occurs in about 2\% of all the possible 2x2x2 contingency tables. It is then interesting to ask why people consider a not-so-uncommon phenomenon psychologically surprising -- the paradoxical feeling appears to suggest some deeply held conviction in people's mind that the Simpson's reversal has clashed with. 

The \emph{sure-thing principle} has been hypothesized to be such a contradictory conviction behind the Simpson's paradox~\cite{2014:pearl}, which validly asserts that a method that helps in every case must be beneficial in terms of the \emph{averaged} performance under any mixture distribution. In the medical example above, for instance, the new method $\M_1$ improves survival rate for both males and females, which by the sure-thing principle does entail that $\M_1$'s average survival rate under any \emph{given} gender ratio must improve. However, it is often overlooked that the \emph{aggregated} survival rate of a method (over both males and females) is \emph{not} a simple average of its per-gender survival rate, but depends on the specific gender ratio that the method is facing (which may vary between methods). 
People might feel the Simpson's reversal paradoxical if they overlooked the difference between the averaged performance and the aggregated performance, in which case the observed reversal clashes with the sure-thing principle in the observer's mind. 

We argue that this often-overlooked disparity between average and aggregate performances, as possibly the real crux behind the Simpson's paradox, \emph{is} indeed sometimes overlooked in the context of NLP training, not only regarding its existence, but also regarding its impact to the training. Given presence of this disparity, a model that is better in terms of averaged per-sample performance could turn out to be worse in terms of the aggregate performance measured by applying the same evaluation metric to the whole data set directly. This reversal in ranking NLP models (or model parameters) can not only lead to biases in the gradient estimation for SGD (which is based on the average performance), causing inefficiency or failure to optimize the model towards better aggregate performance, but more severely, can cause the training to land in sub-optimal solutions (in terms of aggregate performance) even if an oracle optimization procedure is given (which can at its best maximize the average performance). As both the aforementioned issue in model training and the classic Simpson's paradox in statistical sciences are fundamentally rooted from the disparity between two different ways to compute the same metric (averaged or aggregated), we call this disparity, the \emph{Simpson's bias}, so as to highlight the intrinsic connections between the two. 

\if\aaai1
\else
For completion we remark that there is another paradox about the Simpson reversal when we have to make decisions based on the reversed result -- sometimes it feels reasonable to consult the aggregate measurement while in other scenarios the per-case measurement is the one we want to resort to. This is a different paradoxical experience from the ``Simpson's paradox'' that we have discussed above: One occurs when we merely \emph{observe} the reversal, the other occurs when we go on trying to \emph{use} the reversal data. For clarity we will call the former, \emph{Simpson's Reversal Paradox} (SRP), while call the latter, \emph{Simpson's Decision Paradox} (SDP). 
There is an active AI community that study SDP from a causal perspective~\cite{2014:pearl}. Their causal framework also helps explain \emph{why} people often overlooked the Simpson's bias behind SRP. 

We stress, however, that the SDP literature is less relevant to our paper where we focus only on SRP. 
On the other hand, the causal explanation on SRP is complementary to our paper where we point out that the perhaps causally-rooted (or for whatever reason) tendency to overlook the Simpson's bias may not only induce the Simpson's Reversal Paradox in statistical sciences, but may also lead to undesired results in ML/NLP. 
\fi

\section{Experiments}
\label{sec_exp}

This section experimentally studies (1) how \emph{significant} the Simpson's bias can be in standard NLP benchmarks and (2) how the bias \emph{affects} the NLP training in those benchmarks.  In the following, we report observations about these two questions in three common NLP tasks: Paraphrase Similarity Matching (PSM), Named Entity Recognition (NER) and Machine Translation (MT).

\begin{figure*}[t]
	\begin{subfigure}{.25\textwidth}
		\centering
		\includegraphics[trim={0cm 0cm 0cm 0cm},clip,scale=0.14]{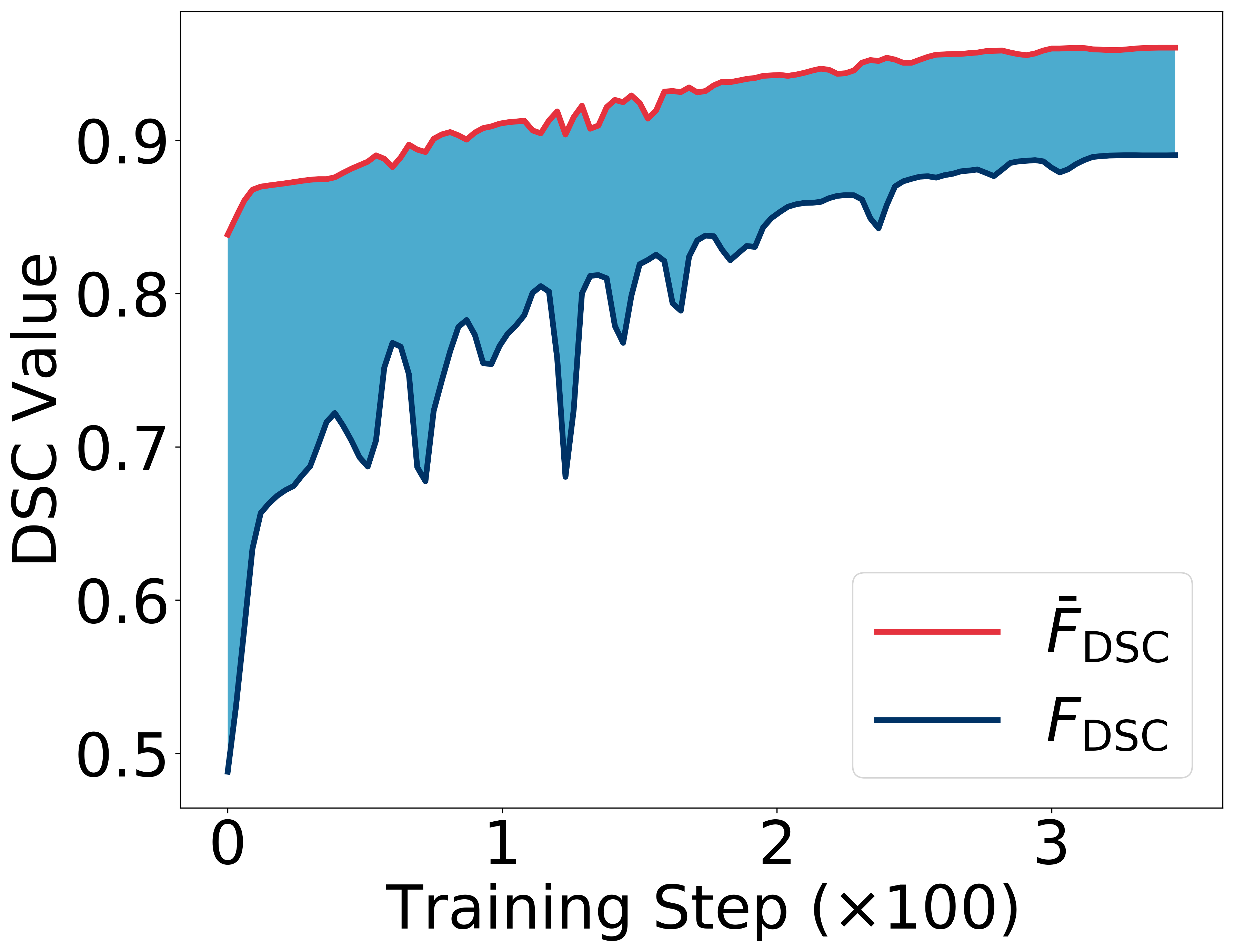}
		\caption{PSM-MRPC}
		\label{fig:1-1}
	\end{subfigure}%
	\begin{subfigure}{.25\textwidth}
		\centering
		\includegraphics[trim={0cm 0cm 0cm 0cm},clip,scale=0.14]{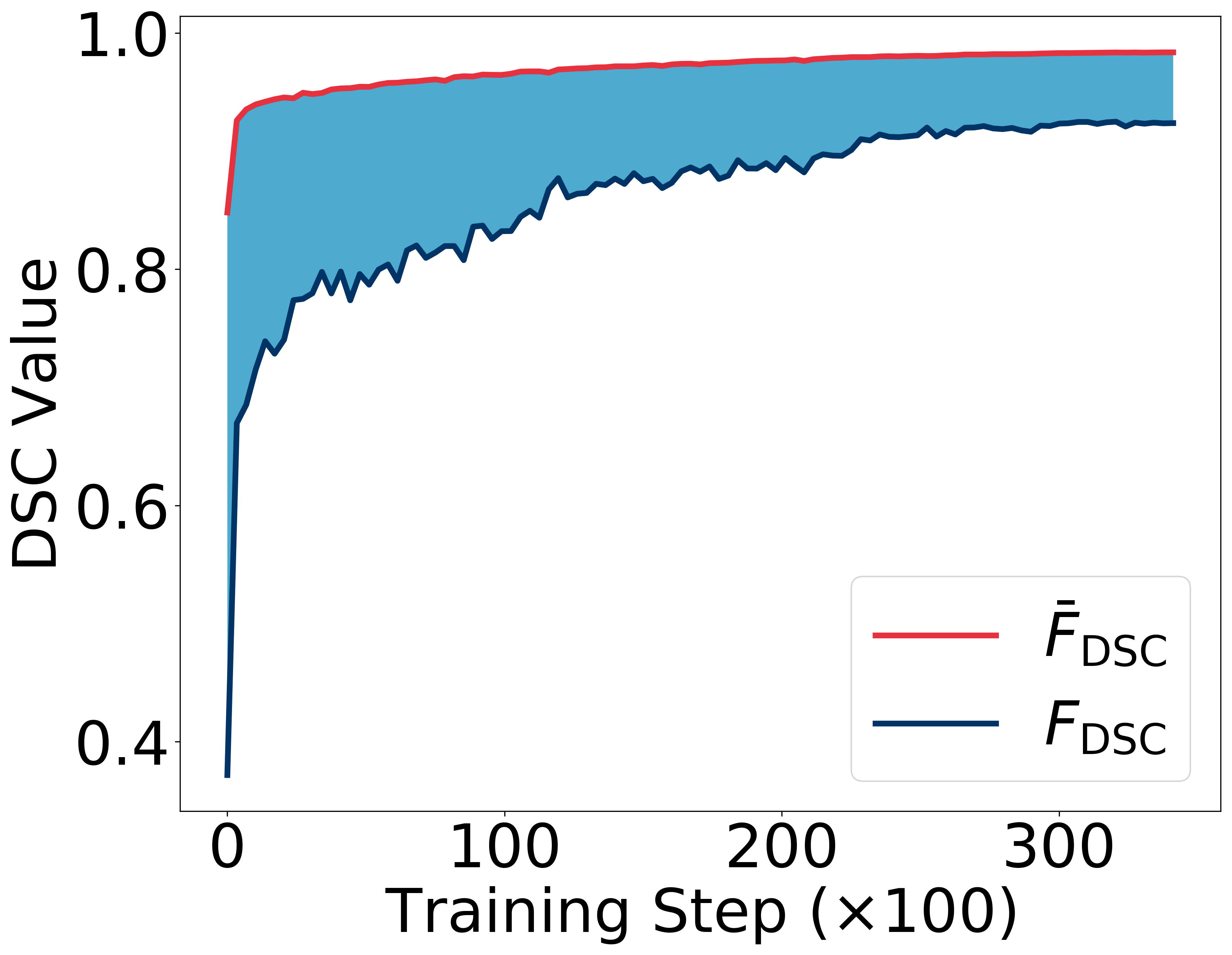}
		\caption{PSM-QQP}
		\label{fig:1-2}
	\end{subfigure}%
	\begin{subfigure}{.25\textwidth}
		\centering
		\includegraphics[trim={0cm 0cm 0cm 0cm},clip,scale=0.14]{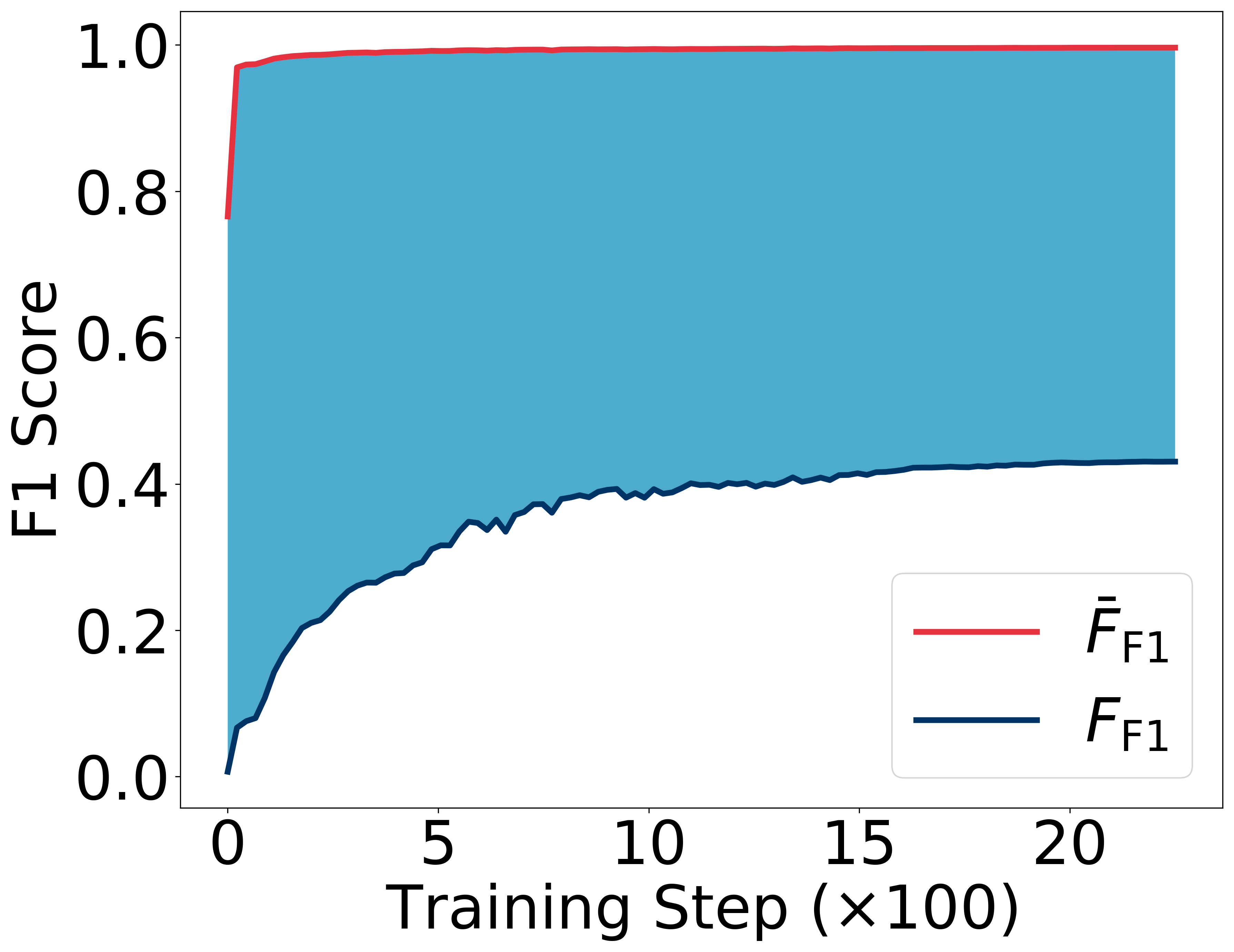}
		\caption{NER}
		\label{fig:1-3}
	\end{subfigure}%
	\begin{subfigure}{.25\textwidth}
		\centering
		\includegraphics[trim={0cm 0cm 0cm 0cm},clip,scale=0.14]{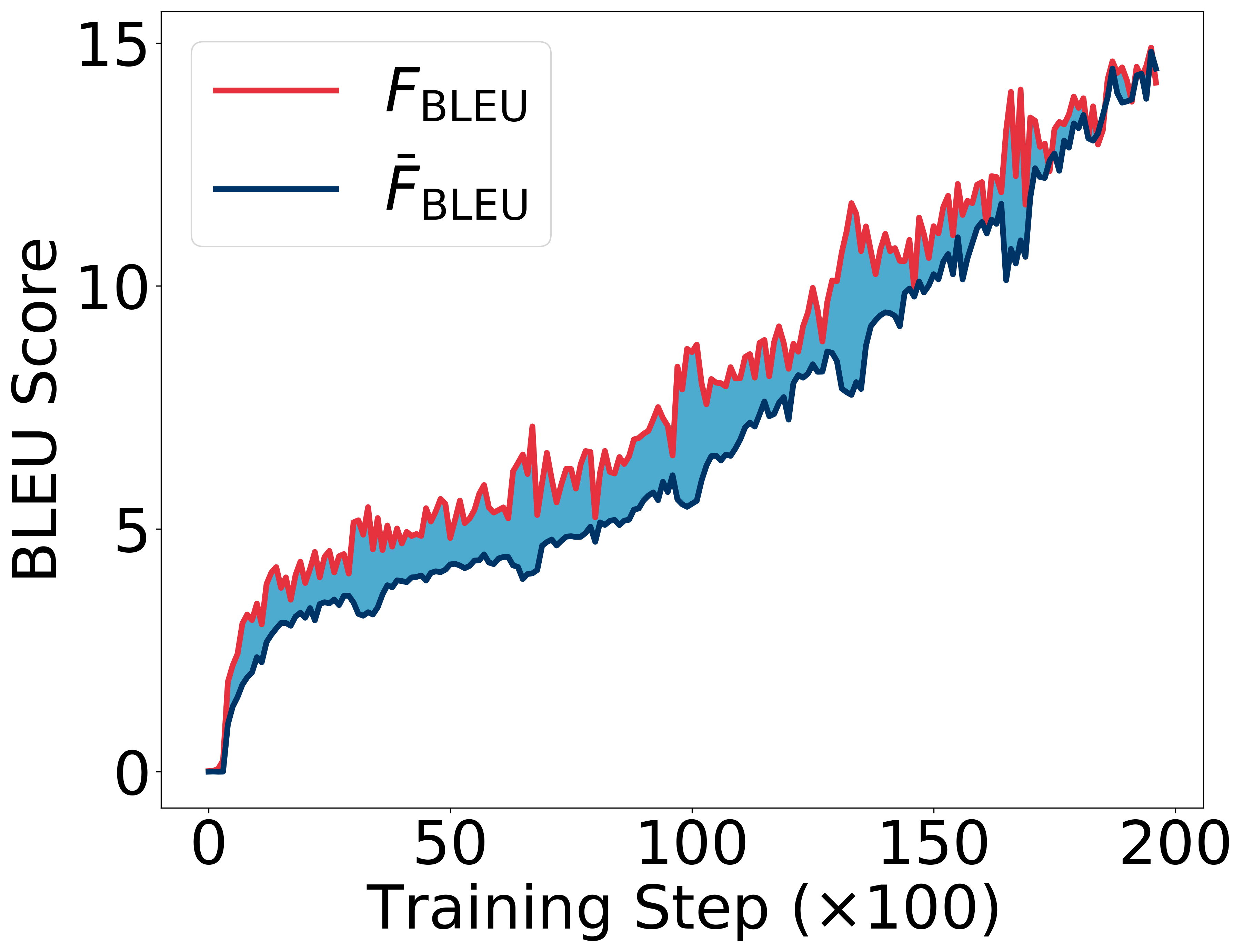}
		\caption{MT}
		\label{fig:1-4}
	\end{subfigure}%
	\caption{The Simpson's bias in NLP training. For PSM or NER task, we observe the Simpson's bias change over time during the model with Dice loss training. On the MT task, we use the model trained by CE loss to observe the bias change. Note that, the model is not necessarily trained with $\Fbar$.}
	\label{fig:simpsons-bias}
\end{figure*}

\subsection{Experiment Design}
The first question is relatively easy to address. Let $\M$ be a NLP model trained for a task with training corpus $S$ and testing metric $\F$, the significance of the Simpson's bias of $\F$ on model $\M$ is denoted by

\begin{equation}
\epsilon(\M) = |\F(S;\M) - \Fbar(S;\M)|
\end{equation}
where $\Fbar$ is the average-form metric corresponding to $\F$. Note that model $\M$ is not necessarily trained with $\Fbar$, but we can generally measure the Simpson's bias between $\F$ and $\Fbar$ on an arbitrary model. In our experiments, we will measure the bias $\epsilon$ in various tasks with various metrics $\F$, and on models trained with various loss functions under various hyper-parameter and pre-processing settings. 

The second question, i.e. to measure the impact of the Simpson's bias, is more tricky. Ideally, one would want to directly compare the performances (in terms of $\F$) between models trained with sample-level objective $\bar{F}$ and those trained with corpus-level objective $\F$. However, a key obstacle here is that we cannot easily compute/estimate the gradient of the corpus-level objective $\F$ (over any corpus beyond modest size) to optimize it, which is exactly why people turned to the sample-level objective $\Fbar$ in the first place. 
In our experiments we instead observe the impact of Simpson's bias to NLP training from three indirect perspectives. 

First, we seek to observe how consistent $\F$ and $\Fbar$ can be when used to compare a given pair of models. Such a model pair essentially serves as a highly degenerate model/parameter space (of size $2$), over which we want to see if the optimum of $\Fbar$ is also the optimum of $\F$. In this paper we focus on comparing pairs of models obtained from consecutive learning steps in a training process. For a learning step $t$, we measure the changing directions at $t$ by calculating the $\Delta\F^t$ and $\Delta\bar{F}^t$ according to:
\begin{equation}
\begin{split}
\Delta\F^t = \F^t - \F^{t-1} \\
\Delta\bar{F}^t = \bar{F}^{t} - \bar{F}^{t-1} 
\end{split}
\end{equation}

The sign of $\Delta\F^t$ or $\Delta\bar{F}^t$ represents the changing direction. $\Delta\F^t \cdot \Delta\bar{F}^t >0$ indicates that $\F$ and $\bar{F}$ are consistent in evaluating the models at $t$ and $t-1$. $\Delta\F^t \cdot \Delta\bar{F}^t \leq 0$ suggests that $\F$ and $\bar{F}$ have changed in opposite directions in step $t$, indicating inconsistent model evaluation. We call such an inconsistent $(\Delta\F^t, \Delta\bar{F}^t)$, a \emph{reversal pair}. If reversal pairs are rare throughout the whole training process, we can say that the changes of $\F$ and $\bar{F}$ are highly consistent. In other words, we can maximize $\F$ by optimizing $\bar{F}$. Alternatively, if there are a large number of reversal pairs, we may at least need a longer time to reach the optimal $\F$. Moreover, a tremendous amount of inconsistent directions increase the risk that $\F$ can be significantly sub-optimal.

Our second experiment to observe the impact of Simpson's bias is to compare models trained with $\Fbar$ to those trained with the standard CE loss. In particular, some previous NLP works, such as ~\citet{DBLP:conf/acl/LiSMLWL20}, proposed to replace the CE loss with smoothed Dice loss for imbalanced data sets due to its similarity to the F1 metric. Instead of asking if models thus trained are competitive to those trained directly with F1, we ask: \emph{How much can the models trained with Dice loss (at least) outperform those with CE loss?} As our theoretical analysis (Theorem \ref{thm_dsc} in particular) has pointed out, optimizing smoothed average-form DSC is actually equivalent to optimize the accuracy. 
One may then expect comparable learning results between smoothed Dice loss and CE loss. If this were indeed the case, it would indirectly indicate that the models trained with Dice loss (corresponding to $\Fbar$) might be substantially sub-optimal in F1 (corresponding to $\F$), assuming that the CE loss (which is not F1-oriented) cannot fully optimize F1 (which was the general premise to consider conjugated loss at all).    

Our third experiment on the impact of Simpson's bias is to examine the correlation between the bias and the training quality (in varying training settings). If high significance-of-bias is correlated with low training quality, it may \emph{potentially} imply some deeper causal relationships between the two.

\subsection{Dataset and Setting}

For PSM, we use two standard data sets: Microsoft Research Paragraph Corpus (MRPC)~\cite{dolan2005automatically} and Quora Question Pairs (QQP)~\cite{wang2018glue}. We adopt the pre-trained BERT-base-uncased model with different training objectives (CE and Dice loss). The officially recommended parameter settings~\cite{Wolf2019HuggingFacesTS} are leveraged, including max sequence length=128, epoch number=3, train batch size=32, learning rate=2e-5, and $\gamma$=1.

For NER, we fine-tune BERT base multilingual cased model with different loss function (CE / Dice) on GermEval 2014 dataset~\cite{benikova2014nosta}. Formally, let $S$ be a NER data set consisting of $n$ sentences in total; each sentence has $L$ tokens. We want to train a neural network model that classifies each token $i$ into one of $K$ predefined entity classes. In the experiment, we use the same setting as \citet{Wolf2019HuggingFacesTS}, including max sequence length=128, epoch=3, lr=5e-5, batch size = 32, $\gamma=1$ and the Dice loss is $1 - \Fbar_{\texttt{F1}}$, where $\Fbar_{\texttt{F1}}$ refers to:
\begin{equation}
\small
    \label{eq:loss_12}
    \begin{split}
        \Fbar_{\texttt{F1}} = \frac{1}{Kn}\sum_i^n\sum_k^K\frac{\sum_j^L{2 \cdot p_{i,j,k} \cdot \indicator{y_i=k}} + \gamma}{\sum_j^L{(p_{i,j,k}^2 + \indicator{y_i=k}^2)} + \gamma}
    \end{split}
\end{equation}

There is an alternative Dice loss $1 - \Fbar'_{\texttt{F1}}$, where $\Fbar'_{\texttt{F1}}$ is defined as:
\begin{equation}
\small
    \label{eq:loss_14}
    \begin{split}
        \Fbar'_{\texttt{F1}} = \frac{1}{KnL}\sum_i^n\sum_j^L\sum_k^K\frac{2 \cdot p_{i,j,k} \cdot \indicator{y_i=k} + \gamma}{p_{i,j,k}^2 + \indicator{y_i=k}^2 + \gamma}
    \end{split}
\end{equation}

Both \eqref{eq:loss_12} and \eqref{eq:loss_14} correspond to dice loss, but \eqref{eq:loss_12} uses the ``standard'' method that classifiers as many entity phrases in a sentence as possible, while \eqref{eq:loss_14} is a variant of \eqref{eq:loss_12} that independently classifies each token, and thus obviously induces the Simpson's bias to \eqref{eq:loss_12}.

This dice loss is in ill condition. Since every sentence in the dataset has not the same number of words, the padding is necessary. Ideally, padding makes no or almost no contribution to the training objective, however, in \eqref{eq:loss_14} the effect of padding is the same as that of negative examples in the dataset without additional processing. At the same time, the smooth strategy is directly applied to each independent token, resulting in the DSC value of a single negative example changing from 0 to 1. Such a change will make training hard.

For MT, we train a transformer model~\cite{DBLP:conf/nips/VaswaniSPUJGKP17} on \texttt{IWSLT 2016} dataset using the default setting in the original paper, except we hold the learning rate constant as $0.0001$ and set the batch size to $10K$ tokens after padding. 

More details of data and setting
\if\aaai1 
appear in the appendix of our full paper.
\else
appear in the appendix.
\fi

\begin{figure*}[t]
	\begin{subfigure}{.25\textwidth}
		\centering
		\includegraphics[trim={0cm 0cm 0cm 0cm},clip,scale=0.14]{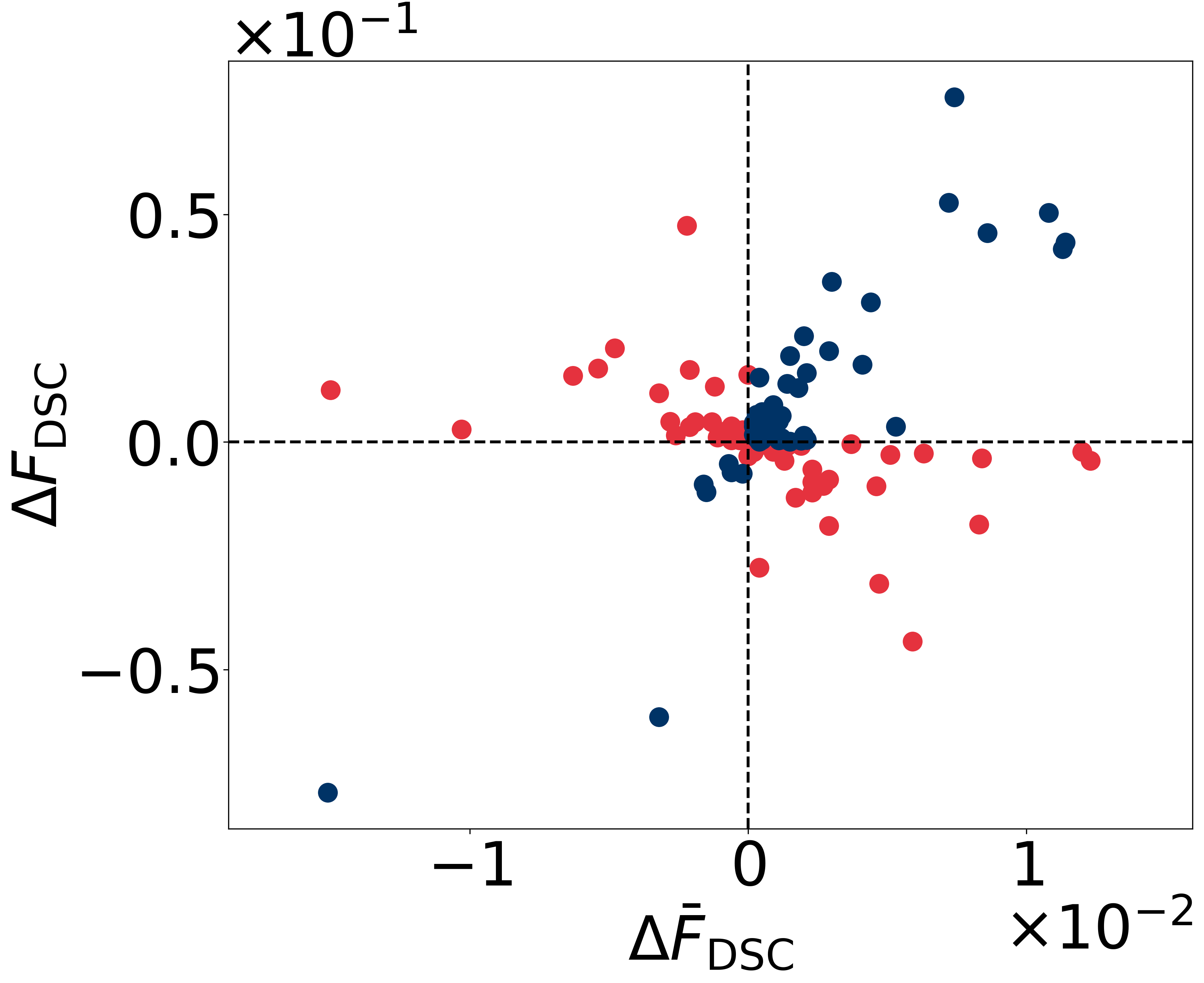}
		\caption{PSM-MRPC}
		\label{fig:2-1}
	\end{subfigure}%
	\begin{subfigure}{.25\textwidth}
		\centering
		\includegraphics[trim={0cm 0cm 0cm 0cm},clip,scale=0.14]{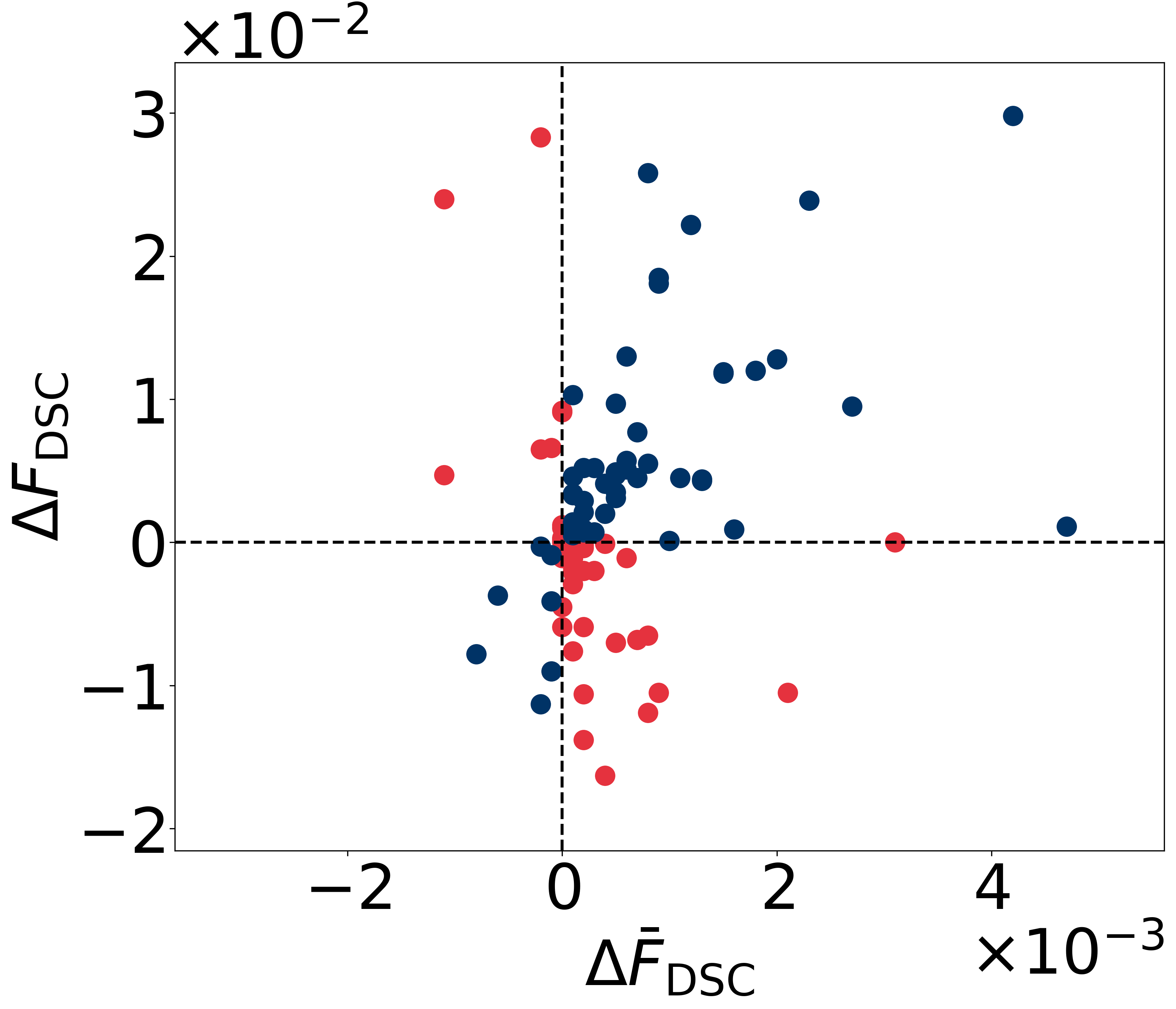}
		\caption{PSM-QQP}
		\label{fig:2-2}
	\end{subfigure}%
	\begin{subfigure}{.25\textwidth}
		\centering
		\includegraphics[trim={0cm 0cm 0cm 0cm},clip,scale=0.14]{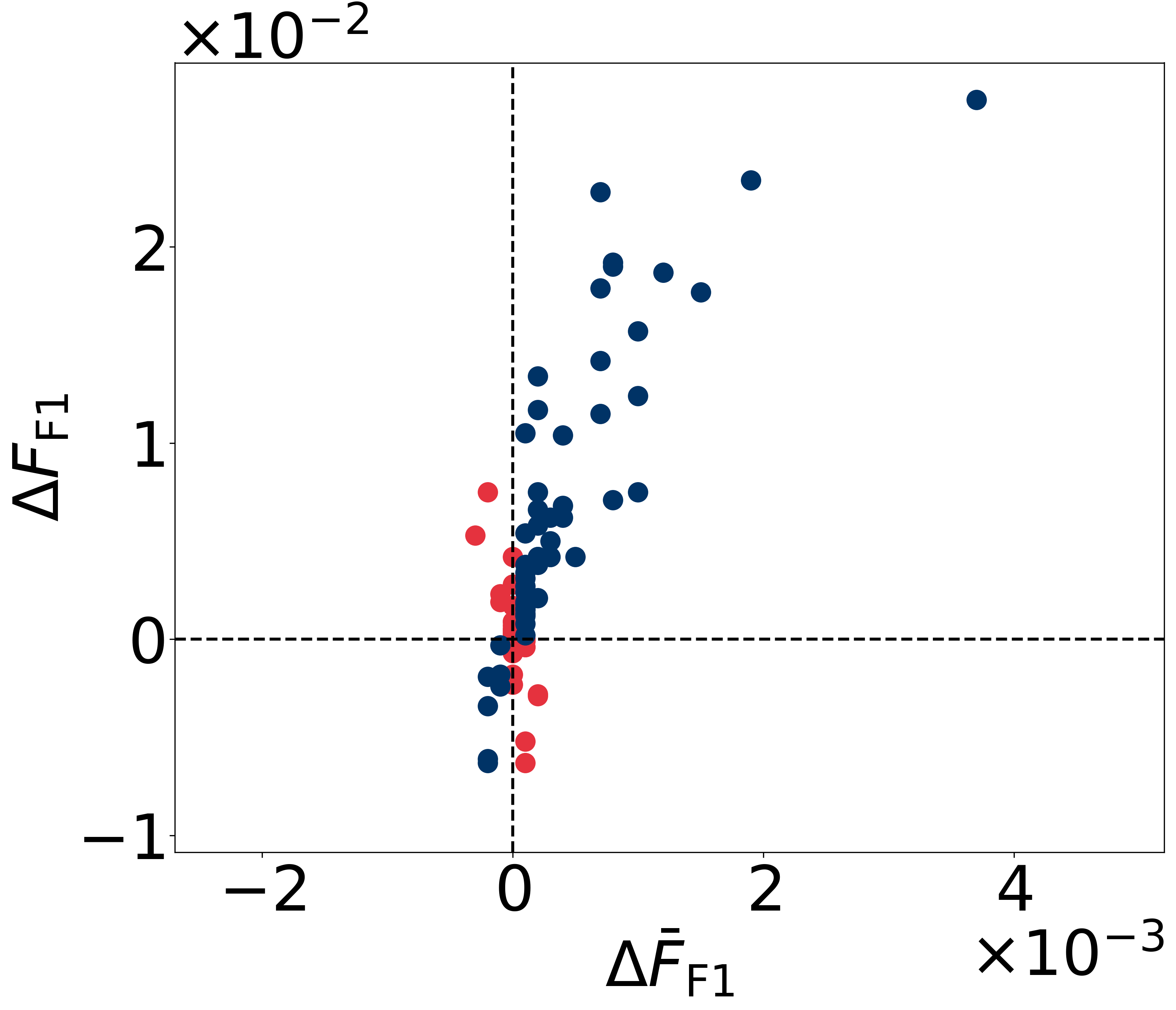}
		\caption{NER}
		\label{fig:2-3}
	\end{subfigure}%
	\begin{subfigure}{.25\textwidth}
		\centering
		\includegraphics[trim={0cm 0cm 0cm 0cm},clip,scale=0.14]{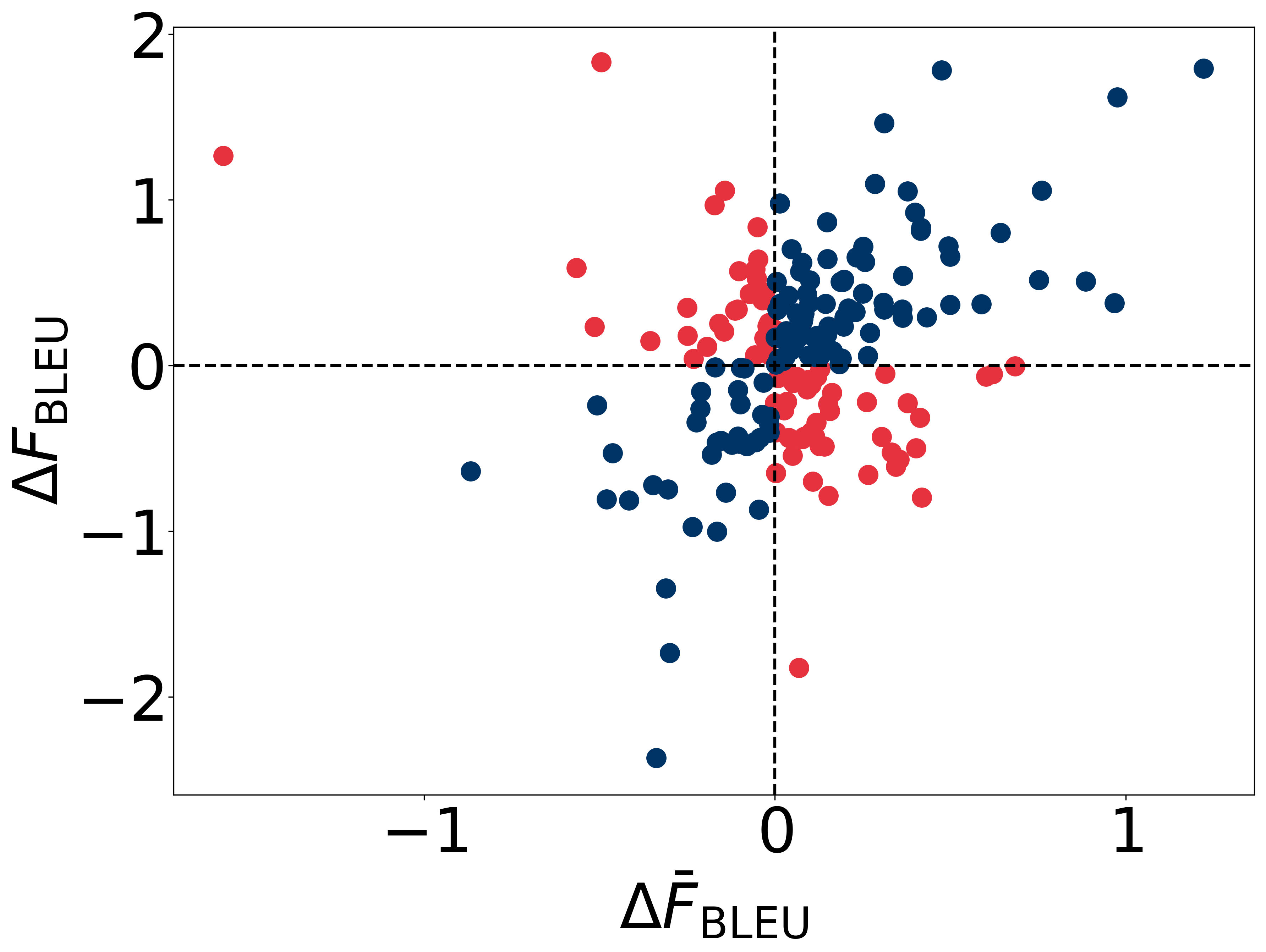}
		\caption{MT}
		\label{fig:2-4}
	\end{subfigure}%
	\caption{Reversal pairs during NLP training.}
	\label{fig:consistency-testing}
\end{figure*}

\subsection{Significance of Simpson's Bias}
For PSM task, Figure \ref{fig:1-1} and Figure~\ref{fig:1-2} show Simpson's bias change overtime during the ``BERT with dice loss($\gamma = 1$)'' training in MRPC/QQP task. As the training progresses, the value gradually decreases, but it still cannot be ignored at the end of the training. For NER task, the Simpson's bias cannot be resolved by $\gamma=1$. Because of the significance of bias between $\Fbar_\texttt{F1}$ and $\F_\texttt{F1}$, it seems $\Fbar_\texttt{F1}$ converges early in Figure~\ref{fig:1-3}, but it is not. Actually, in the whole training process, $\Fbar_\texttt{F1}$ increases rapidly and then changes with small-scale. At this time, $\F_\texttt{F1}$ increases slowly and finally converges to about 0.4. For MT task, Figure~\ref{fig:1-4} shows the changes of the $\F_{\texttt{BLEU}}$ and $\Fbar_{\texttt{BLEU}}$ scores over time during training. As they both increase, we can see clear disparity between them. Through these observations, we find that (1) smooth strategy in these NLP tasks is of limited use for eliminating bias; (2) during the whole training process, the value of bias is significant and cannot be ignored.

\subsection{Impact of Simpson's Bias}
\subsubsection{Consistency testing}
This experiment seeks to observe how consistent $\F$ and $\Fbar$ can be when used to compare a given pair of models. For PSM task, Figure \ref{fig:2-1} and \ref{fig:2-2} show a clear inconsistency between the changes in $F_\texttt{DSC}$ and $\bar{F}_\texttt{DSC}$ on MRPC and QQP task. By tracking the tendency of the DSC value changes at the $F_\texttt{DSC}$ and $\bar{F}_\texttt{DSC}$, we find out of the $115$ training steps, $59$ (or half of them) show an opposite trends between $\Delta\F_\texttt{DSC}$ and $\Delta\bar{F}_\texttt{DSC}$. 46 out of 100 sample dots pairs in Figure~\ref{fig:2-2} has different change directions, the red dots indicate the disparity between $\Delta\F_\texttt{DSC}$ and $\Delta\bar{F}_\texttt{DSC}$. For NER task, there some extreme values in model early training, which reflect the fastest improvements. But the existence of these extreme values hinder our analysis, so it does not exist in Figure~\ref{fig:2-3}. It can be seen from Figure~\ref{fig:2-3}, in most cases, the change directions of $\bar{F}_\texttt{F1}$ and $\F_\texttt{F1}$ are completely inconsistent. For MT task,  we plotted the scattered dots for each $(\Delta\F_{\texttt{BLEU}}, \Delta\Fbar_{\texttt{BLEU}})$ pairs to see whether they both increase or decrease in the same direction. There are $77$ / $195$ sampled dots have different changing directions in total. There are a larger number of reversal pairs on these NLP tasks,  $\F$ may at least need a longer time to reach the optimal. Moreover, the high degree of inconsistency between $\bar{F}$ and $\F$ may increase the difficulty for $\F$ optimization.

\subsubsection{Comparison with CE}

This experiment is to observe the impact of Simpson's bias by comparing models trained with $\Fbar$ to those trained with the standard CE loss. For PSM task, as show in Table~\ref{tab:comparison-with-CE}, BERT trained with the CE loss (a.k.a. $\Fbar_{\texttt{MLE}}$) outperforms the model parameters trained with Dice loss (i.e., BERT + Dice) by a small margin: + 0.78/0.45 in terms of F1 score on MRPC/QQP task. For NER task, as the Table~\ref{tab:comparison-with-CE} shows, the model trained with CE is about 3.53 point higher than that trained with Dice. All the result in Table \ref{tab:comparison-with-CE} indicates the fact that the Dice did not achieve better performance may suggest that it does not necessarily drive the optimization toward high DSC scores, despite of their similarity. And using smoothing constants $\gamma \in [0,1]$ does not work to eliminate Simpson's bias on these tasks. 

\begin{table}[h]
	\centering\small
	\begin{tabular}{lcccc}  
		\toprule
		Loss & MRPC & QQP & NER \\
		\midrule
		CE Loss & 89.78 & 87.84 & 86.14 \\
		Dice Loss & 89.00 & 87.39 & 82.61 \\
		\bottomrule
	\end{tabular}
	\caption{Performance(F1 Score) of various training objective on dev set for MRPC/QQP task, and test set for NER task.}
	\label{tab:comparison-with-CE}
\end{table}

\subsubsection{Impacts on training quality}

We conduct more experiments under different settings to get various $\bar{F}$ variant on MRPC task. No matter how to modify the hyper-parameter, this bias between $\F$ and $\Fbar$ is still significant, there are still a lot of reversed pairs and the performance of the model trained with $\Fbar$ is worse than that of CE. Meanwhile, we find a negative relation between the model quality on train dataset $\mathrm{F1}^\texttt{Dice}_\texttt{train}$ and the significance of bias $\epsilon$. Figure~\ref{fig: bias-vs-quality} is a scatter plot that shows the significance of bias and training quality. As can be seen from the figure, when $\mathrm{F1}^\texttt{Dice}_\texttt{train}$ tends to decrease as $\epsilon$ increases. These experiments results suggest that the Simspon's bias is a common phenomenon in NLP training and not changing with model tuning. 
\if\aaai1
See more discussions in appendix of our full paper.
\else
See more discussions in appendix.
\fi

\begin{figure}[h]
    \centering
	\includegraphics[trim={0cm 1.5cm 0cm 0cm},clip,scale=0.14]{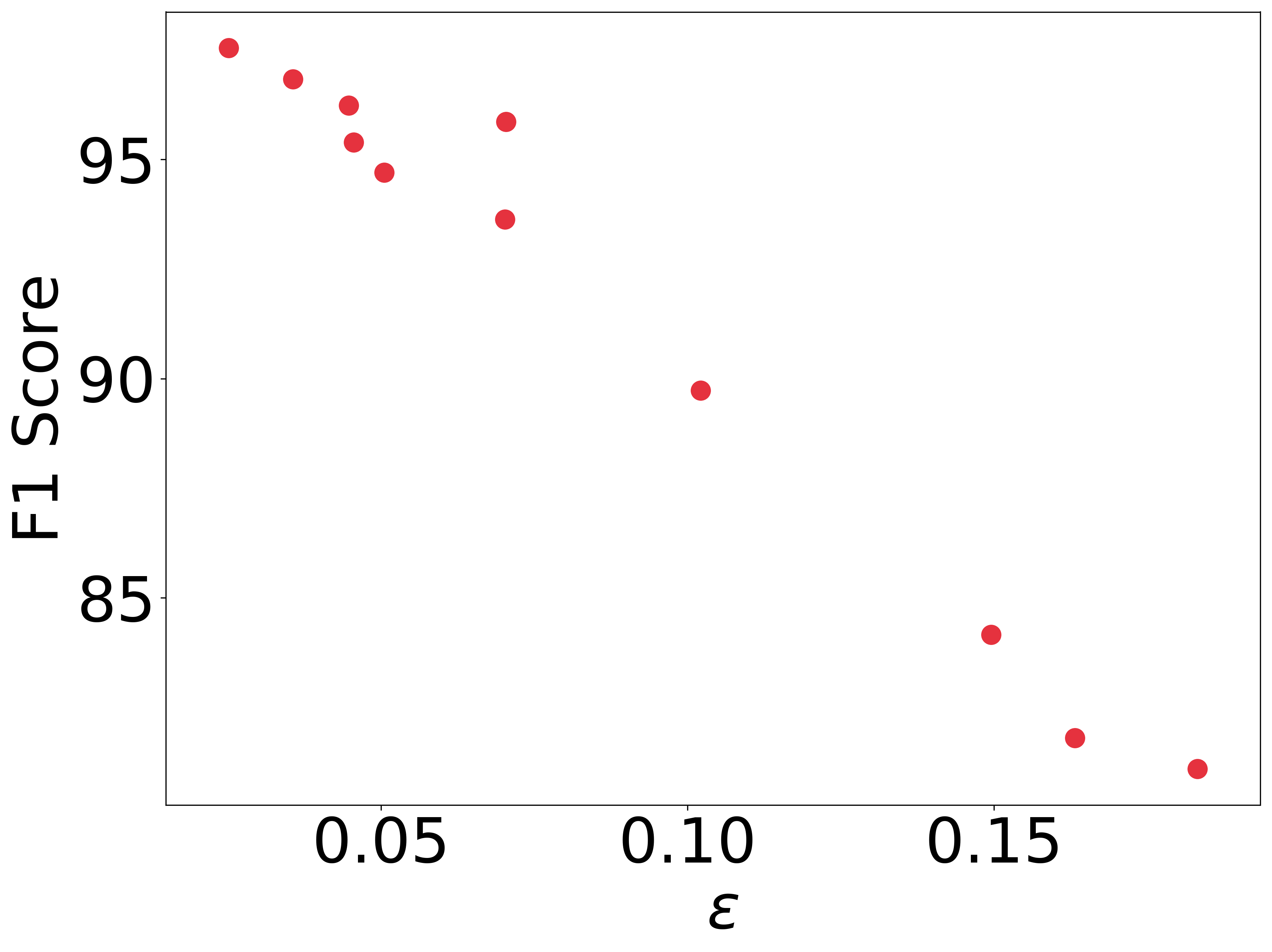}
	\caption{Significance of bias $\epsilon$ vs $\mathrm{F1}^\texttt{Dice}_\texttt{train}$.}
	\label{fig: bias-vs-quality}
\end{figure}%

\section{Conclusions}

In this paper we coined a new concept, Simpson’s bias, for its similar role in inducing sub-optimal training in ML and in inducing the Simpson’s paradox in statistics. We presented a theoretical taxonomy for the Simpson’s bias in ML, revealing how similar effect is embodied in a wide spectrum of ML metrics, from ones as simple as Accuracy, to ones as sophisticated as BLEU. For \emph{some} aggregate-form metrics, we show that it is possible to construct \emph{provably unbiased} average-form surrogate through adding special and uncommon (e.g. negative) smoothing constants. But the Simpson's bias is generally a factor with important impact in a variety of NLP tasks, as our experiments showed. We observed both noticeable margins of the bias and a significant number of “reversed” SGD steps in all the different tasks, data-sets, and metrics. 
Our experiments also show that models trained with “naively-conjugated” objectives (such as dice loss to F1) can be even worse than those trained with non-conjugated objectives (such as CE loss to F1), which could potentially reflect a significant sub-optimality when training using (seemingly-)conjugated objectives. Finally, a clear correlation between the Simpson's bias and training quality is consistently observed. We believe these results indicate that the Simpson's bias is a serious issue in NLP training, and probably in machine learning in general, that deserves more studies in the future.

\bibliographystyle{aaai}
\bibliography{simpson_bias}

\if\aaai1
\else
\clearpage
\appendix

\section{Proofs}
\label{sec_proof}
Both Theorem \ref{thm_precision_1} and \ref{thm_precision_2} are based on the following lemma.
\begin{lemma} \label{lem_precision}
	$\Fbar_{\texttt{P}^\gamma}(S) = 1 - \frac{ |\{\hat{y}_i=1, y_i=0\}| }{ (1+\gamma) n }$ for $\gamma \neq 0, -1$.
\end{lemma}
\begin{proof}
	By definition
	$\F_{\texttt{P}^\gamma}(i) = \frac{\gamma+y_i\hat{y}_i}{\gamma+\hat{y}_i}$. As both $\hat{y}_i$ and $y_i$ are binary variables in $\{0,1\}$, we can write the contingency table of $\F_{\texttt{P}^\gamma}(i)$ as
	
	\begin{center}
	\begin{tabular}{cc|cl}
		\hline 
		\rule[-1ex]{0pt}{2.5ex} $\hat{y}_i$ & $y_i$ & $\F_{\texttt{P}^\gamma}(i)$ & \\ 
		\hline 
		\rule[-1ex]{0pt}{2.5ex} $0$ & $0$ & $1$ & for $\gamma \neq 0$ \\ 
		\rule[-1ex]{0pt}{2.5ex} $0$ & $1$ & $1$ & for $\gamma \neq 0$\\ 
		\rule[-1ex]{0pt}{2.5ex} $1$ & $0$ & $\frac{\gamma}{1+\gamma}$ & for $\gamma \neq -1$ \\ 
		\rule[-1ex]{0pt}{2.5ex} $1$ & $1$ & $1$ & for $\gamma \neq -1$ \\ 
		\hline
	\end{tabular}
	\end{center}

	\noindent 
	from which we see that $\F_{\texttt{P}^\gamma}(i)$ is anchored at $1$ except for $\hat{y}_i=1$ and $y_i=0$ in which case $\F_{\texttt{P}^\gamma}(i)$ gets an additional penalty of $\frac{1}{1+\gamma}$. With this observation we immediately have 
	\eqmx{
		\Fbar_{\texttt{P}^\gamma}(S) 
		&= 	\frac{1}{n} \sum_{i=1}^n \F_{\texttt{P}^\gamma}(i)
		\\&=\frac{1}{n}\Big( n  - \sum_{i\in \{\hat{y}_i=1, y_i=0\}} \frac{1}{1+\gamma} \Big)
		\\&= 1 - \frac{\big| \{\hat{y}_i=1, y_i=0\} \big|}{n(1+\gamma)}
	}
\end{proof}

\begin{proofof}{Theorem \ref{thm_precision_1}}
	Let $\texttt{FP} \define \{\hat{y}_i=1, y_i=0\}$ and $\texttt{TP} \define \{\hat{y}_i=1, y_i=1\}$ denote the set of false positives and true positives, respectively. From Lemma \ref{lem_precision} we have $\Fbar_{\texttt{P}^\gamma}(S)=1-\frac{|\texttt{FP}|}{n(1+\gamma)}$. On the other hand, from \eqref{precision_gamma} we have $\F_{\texttt{P}^\gamma}(S) = \frac{|\texttt{TP}|+\gamma}{|\texttt{TP}|+|\texttt{FP}|+\gamma} = 1 - \frac{|\texttt{FP}|}{|\texttt{TP}|+|\texttt{FP}|+\gamma}$. Comparing the two equations we see that $\Fbar_{\texttt{P}^\gamma}(S) = \F_{\texttt{P}^\gamma}(S)$ when the denominators are equal, that is, if
	\eq{thm_precision_1_eq1}{
		n + n\gamma = |\texttt{TP}|+|\texttt{FP}|+\gamma
	.}
	Rearranging \eqref{thm_precision_1_eq1} gives $\gamma = \frac{|\texttt{P}|-n}{n-1} = \frac{\sum_i \hat{y}_i-n}{n-1}$ as desired. 
	
	Note that \eqref{thm_precision_1_eq1} is based on Lemma \ref{lem_precision} which requires $\gamma \neq 0$ and $-1$, or equivalently, requires $\sum_i \hat{y}_i \neq n$ and $1$. As the theorem has excluded the case of $\sum_i \hat{y}_i=1$, we only need to further encompass the special case of $\sum_i \hat{y}_i=n$. 
	
	The problem with $\sum_i \hat{y}_i=n$ is that in this case $\gamma = \frac{\sum_i \hat{y}_i-n}{n-1}=0$, and that $\gamma=0$ invalidates Lemma \ref{lem_precision}. However, having a closer look at its proof we see that the whole reason for Lemma \ref{lem_precision} to exclude $\gamma=0$ is exactly because $\gamma=0$ makes the first two entries of $\F_{\texttt{P}^\gamma}$'s contingency table ill-defined. Nevertheless, note that with $\sum_i \hat{y}_i=n$ we are dealing with a special model that always outputs $\hat{y}_i\equiv 1$, in which case we never run into the first two entries of $\F_{\texttt{P}^\gamma}$'s contingency table at all. As a result, in the special case of $\sum_i \hat{y}_i=n$, Lemma \ref{lem_precision} holds -- and thus \eqref{thm_precision_1_eq1} also holds -- even if $\gamma=0$.
	
	Finally, we remark that for $\sum_i \hat{y}_i=1$, that is, for models with exactly one positive output throughout the data set $S$, we indeed must have $\gamma\neq -1$ otherwise $\F_{\texttt{P}^\gamma}$ is ill-defined on that single positive instance. On the other hand, we see from the above proof that $\Fbar_{\texttt{P}^\gamma}(S) = \F_{\texttt{P}^\gamma}(S)$ only if $\gamma = \frac{\sum_i \hat{y}_i-n}{n-1} = -1$. The contradiction means there is no way to make $\Fbar_{\texttt{P}^\gamma}(S) = \F_{\texttt{P}^\gamma}(S)$ when $\sum_i \hat{y}_i=1$.
\end{proofof}

\begin{proofof}{Theorem \ref{thm_precision_2}}
	The proof idea is similar to that for Theorem \ref{thm_precision_1} except that now we want to connect $\Fbar_{\texttt{P}^\gamma}(S)=1-\frac{|\texttt{FP}|}{n(1+\gamma)}$ to $\F_{\texttt{P}}(S) = \F_{\texttt{P}^0}(S) = 1 - \frac{|\texttt{FP}|}{|\texttt{TP}|+|\texttt{FP}|}$. Clearly, the equality condition for $\Fbar_{\texttt{P}^\gamma}(S) = \F_{\texttt{P}}(S)$ is
	\eq{thm_precision_2_eq1}{
		n + n\gamma = |\texttt{TP}|+|\texttt{FP}|
	}
	or equivalently, $\gamma = \frac{|\texttt{P}|-n}{n} = \frac{\sum_i \hat{y}_i-n}{n}$. 
	
	Again, we need to discuss the two special cases $\sum_i \hat{y}_i=n$ and $\sum_i \hat{y}_i=0$ separately (as $\gamma=0$ and $-1$ in these two cases, respectively, which invalidate Lemma \ref{lem_precision}). But this time we observe that Theorem \ref{thm_precision_2} is valid in both special cases, so we don't need to exclude any model (even those that always or never output positive) from the theorem. 
	Specifically, when $\sum_i \hat{y}_i=0$ (or $\sum_i \hat{y}_i=n$) we have $\hat{y}_i\equiv 0$ (or $\hat{y}_i\equiv 1$), in which case Lemma \ref{lem_precision} and \eqref{thm_precision_2_eq1} hold even for $\gamma=0$ (or $\gamma=-1$), as the last (or first) two entries of $\F_{\texttt{P}^\gamma}(i)$'s contingency table are impossible. 
\end{proofof}

\begin{proofof}{Theorem \ref{thm_dsc}}
	The proof idea is similar to that of Lemma \ref{lem_precision}. By definition $\F_{\texttt{DSC}^\gamma}(i) = \frac{\gamma+2y_i\hat{y}_i}{\gamma+y_i+\hat{y}_i}$, whose contingency table is as follows.

    \begin{center}
	\begin{tabular}{cc|cl}
		\hline 
		\rule[-1ex]{0pt}{2.5ex} $y_i$ & $\hat{y}_i$ & $\F_{\texttt{DSC}^\gamma}(i)$ & \\ 
		\hline 
		\rule[-1ex]{0pt}{2.5ex} $0$ & $0$ & $1$ & for $\gamma \neq 0$ \\ 
		\rule[-1ex]{0pt}{2.5ex} $0$ & $1$ & $\frac{\gamma}{1+\gamma}$ & for $\gamma \neq -1$\\ 
		\rule[-1ex]{0pt}{2.5ex} $1$ & $0$ & $\frac{\gamma}{1+\gamma}$ & for $\gamma \neq -1$ \\ 
		\rule[-1ex]{0pt}{2.5ex} $1$ & $1$ & $1$ & for $\gamma \neq -2$ \\ 
		\hline
	\end{tabular}
    \end{center}
    
	\noindent 
	from the table we see that $\F_{\texttt{DSC}^\gamma}(i)=1$ when $y_i=\hat{y}_i$ and $\F_{\texttt{DSC}^\gamma}(i)=1 - \frac{1}{1+\gamma}$ when $y_i\neq \hat{y}_i$. With this observation we have 
	\eqmx{
		\Fbar_{\texttt{DSC}^\gamma}(S) 
		&= 	\frac{1}{n} \sum_{i=1}^n \F_{\texttt{DSC}^\gamma}(i)
		\\&=\frac{1}{n}\Big( n  - \sum_{i\in \{\hat{y}_i\neq y_i\}} \frac{1}{1+\gamma} \Big)
		\\&= 1 - \frac{\big| \{\hat{y}_i\neq y_i\} \big|}{n(1+\gamma)}
	}	
	when $\gamma \neq 0,-1,-2$.
	
	Note that the above result also implies that with $\gamma\approx 0$ (such as $\gamma=10^{-6}$), we have $\F_{\texttt{DSC}^\gamma}(i) \approx 0$ when $y_i\neq \hat{y}_i$, and $\F_{\texttt{DSC}^\gamma}(i)=1$ when $y_i=\hat{y}_i$. In other words, in this case we have $\F_{\texttt{DSC}^\gamma}(i) \approx \F_{\texttt{AC}}(i)$, which in turn means that $\F_{\texttt{DSC}^\gamma}(S) = \F_{\texttt{AC}}(S)$ for $\gamma\approx 0$.
\end{proofof}

\begin{table*}[t]
    \centering
    \begin{tabular}{cccccccc}  
    \toprule
lr &	 bs  &	model &	 $F1^{Dice}_{train}$ &	 $\epsilon$  &	 $F1^{Dice}_{dev}$  &	 $R$ & $R$ ratio \\
\midrule
2.00E-05 &	16  &	BERT	&   96.23	&   0.0447	&   90.18	&   50   &  51.02 \\
2.00E-06 &	16	&   BERT	&   81.81	&   0.1632	&   81.93	&   34   &  34.69 \\
5.00E-05 &	16	&   BERT	&   89.73	&   0.1021	&   87.93	&   53   &  54.08 \\
\hdashline
2.00E-05 &	32	&   RoBERTa &   94.70	&   0.0505	&   91.70	&   31   &  28.70 \\
2.00E-05 &	32	&   M-BERT 	&   95.86	&   0.0704	&   88.54	&   28   &  26.17 \\
2.00E-05 &	32	&   BERT	&   93.63	&   0.0702	&   89.00	&   55   &  48.25 \\
2.00E-06 &	32	&   BERT	&   81.10	&   0.1832	&   81.94	&   31   &  27.19 \\
5.00E-05 &	32	&   BERT	&   96.83	&   0.0356	&   90.31	&   45   &  39.47 \\
\hdashline
2.00E-05 &	8	&   BERT	&   95.39	&   0.0455	&   90.69	&   43   &  43.88 \\
2.00E-06 &	8	&   BERT	&   84.16	&   0.1495	&   83.25	&   51   &  52.04 \\
5.00E-05 &	8	&   BERT	&   97.54	&   0.0251	&   89.27	&   36   &  36.73 \\
\bottomrule
\end{tabular}

\caption{Experiment results from model with various learning rate (lr), training batch size (bs) and different models including: BERT (bert-base-uncased), RoBERTa (roberta-base) and M-BERT (bert-base-multilingual-uncased) on MRPC task. $F1^{Dice}_{dataset}$ refers to a F1 from a model trained with Dice to perform inference on $(train/dev)$ dataset. $\epsilon$ represents the significance of bias on a model trained with Dice. $R$ refers to reversal pair and $R$ ratio = \#$R$ / \#total step.}
\label{tab:hyper-para}
\end{table*}

\section{Data and Setting}

\subsubsection{Paraphrase Similarity Matching (PSM)}   For PSM, we use two standard data sets, Microsoft Research Paragraph Corpus (MRPC)~\cite{dolan2005automatically}, and Quora Question Pairs (QQP)~\cite{wang2018glue}. The purpose of MRPC is to determine whether two sentences in the pair are semantically equivalent. For MRPC task, there are 4k sentence-pairs in taining corpus and 408 data in development set. The goal of QQP is to determine whether two questions are semantically equivalent. For both two datasets, we use Accuracy and F1(DSC) Score to evaluate the model performance. For QQP task, there are 364k sentence-pairs in training set and 40k sentence-pairs in development set. As discussed above, the Dice loss originates from DSC (F1) value. 

\subsubsection{Named Entity Recognition (NER)} 
Name Entity Recognition (NER) is a popular task in NLP. NER's goal is to identify and segment the named entities, then classify them under various predefined classes, such as a person, location, and organization. We fine-tune BERT base multilingual cased model with different loss function (CE / Dice) on GermEval 2014 dataset~\cite{benikova2014nosta}. This dataset builds on German Wikipedia and News Corpora, which covers over 31,000 sentences corresponding to over 590,000 tokens. After filtering ``control character'', we split longer sentences into smaller ones (once the max sequence length is reached) and generate their label list file, containing all available labels in the dataset. Since the data is a sample from the German language, we use the multilingual version of BERT, released by ~\citet{devlin2018bert} as a language model pre-trained by monolingual corpora in 104 languages. This model excels at zero-shot cross-lingual tasks.

\subsubsection{Machine Translation (MT)}
Machine translation (MT) task is to transform source language text into target language text which has similar meanings. The most popular MT models $f$ are trained on collections of source and target sentence pairs of similar meaning $(X^{(i)},Y^{(i)})$ which are tokenized sentences in each language. Specifically, let $X=(\texttt{BOS}, x_1, x_2, \dots, x_N, \texttt{EOS})$ be the tokenized source sentence of length $N$ and $Y=(\texttt{BOS}, y_1, y_2, \dots, y_M, \texttt{EOS})$ be the tokenized target of length $M$, where $x_i$ and $y_i$ are tokens in source and target vocabulary, the model learns to predict $Y$ based on $X$. At each time step $t$, the model make predictions of the next target token based on $X$ and partial $Y_{1:t-1}=(\texttt{BOS}, y_1, y_2, \dots, y_{t-1})$ by outputting a probability over target vocabulary $y_t=f(X, Y_{1:t-1})$. 
During training, usually the ground truth $y_t$ is added to $Y_{1:t-1}$ to make the next prediction. To make correct predictions, the model seeks to maximize the likelihood probability of $Y$ on each token. During testing, the beam search algorithm will search for the best tokens until the $\texttt{EOS}$ marker is reached. Usually, the corpus-level BLEU~\cite{papineni2002bleu, post-2018-call} metric is used to evaluate the model performance. We use \texttt{IWSLT 2016} dataset of English--German language pair as our training data, which contains $209678$ sentence pairs.





\section{Significance of Bias vs Training Quality}

It can be seen from Table~\ref{tab:hyper-para} that when the hyper parameters are fixed, the larger the value of $\epsilon$, the smaller the end-metric (F1) on the train set. The results on the dev set are slightly different, which may cause by model generalization. There is no significant agreement between the value of $\epsilon$ and the number or reversal pairs ratio. This may be because the bias is not only related to the number of reversal pairs, but also to the degree of reversal. These experimental results suggest that the Simpson's bias is a common phenomenon and not changing with model tuning.

\begin{figure*}[t]
\begin{subfigure}{.5\textwidth}
\centering
  \includegraphics[trim={0cm 0cm 0cm 0cm},clip,scale=0.20]{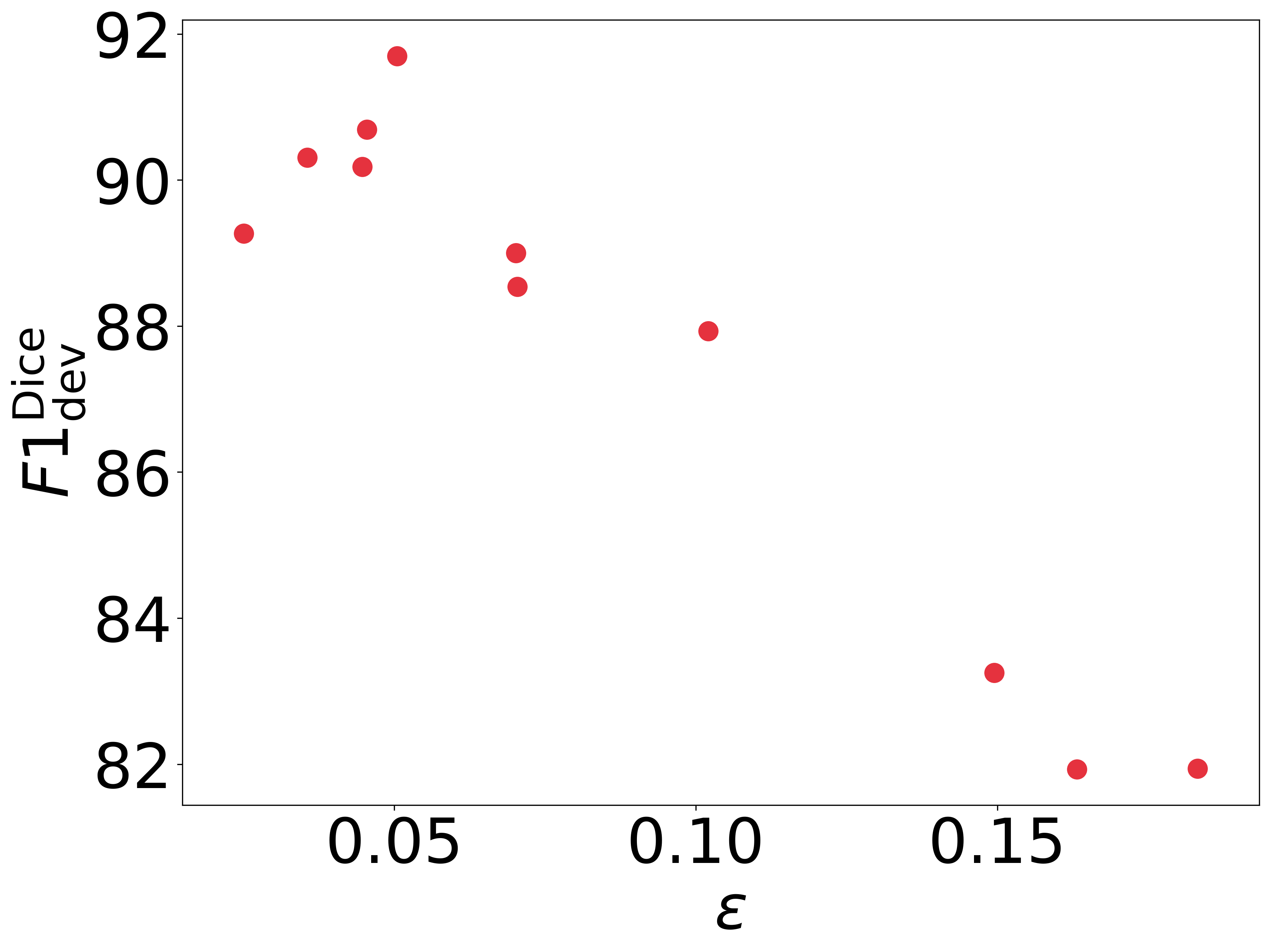}
	\caption{bias $\epsilon$ vs $\mathrm{F1}^\texttt{Dice}_\texttt{dev}$}
  \label{fig:3-1}
\end{subfigure}%
\begin{subfigure}{.5\textwidth}
\centering
  \includegraphics[trim={0cm 0cm 0cm 0cm},clip,scale=0.20]{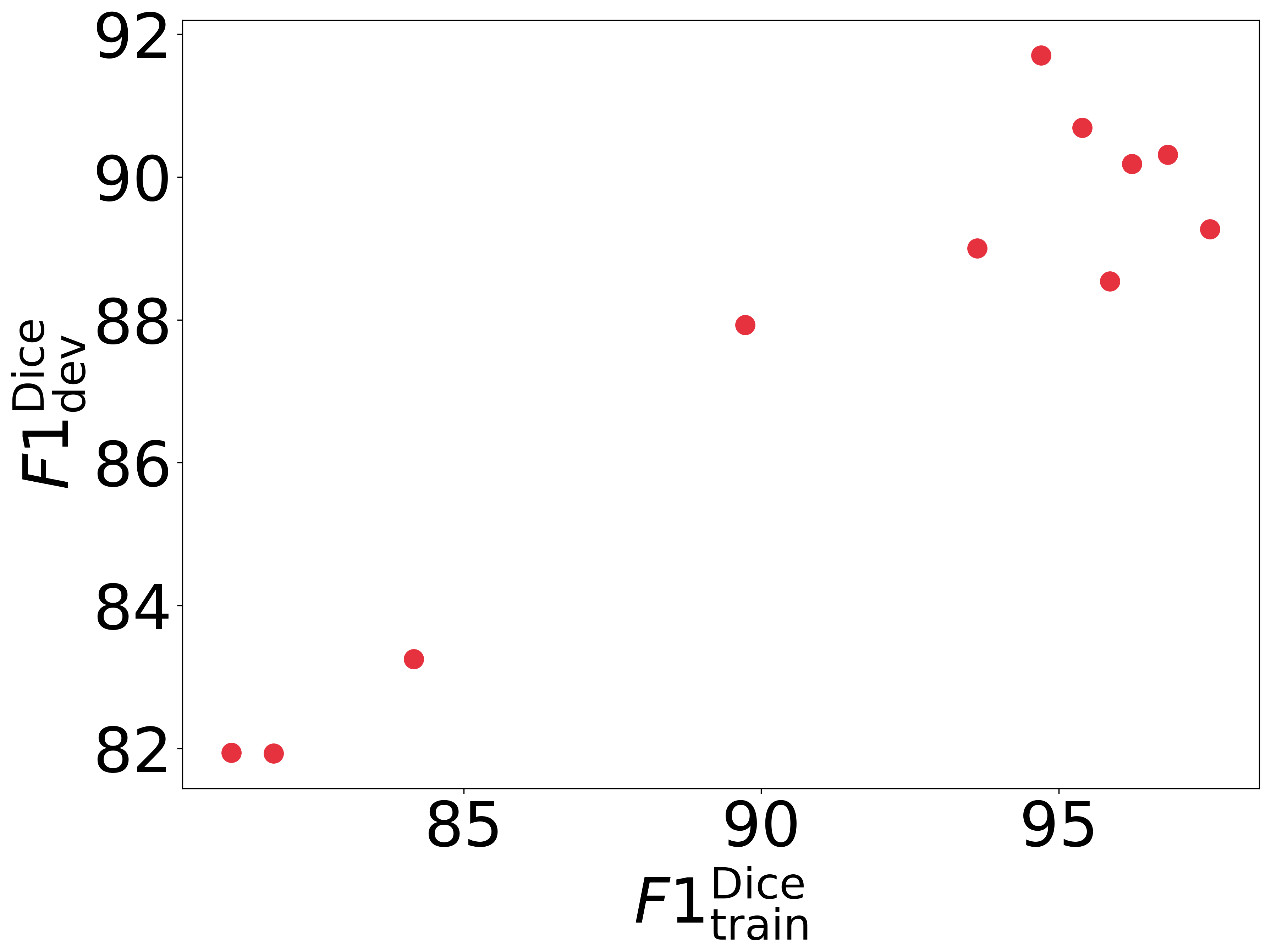}
	\caption{$\mathrm{F1}^\texttt{Dice}_\texttt{train}$ vs $\mathrm{F1}^\texttt{Dice}_\texttt{dev}$}
  \label{fig:3-2}
\end{subfigure}%
\caption{The impacts of the Simpson's bias on training quality.}
\end{figure*}

Intuitively, batch size may have a special impact on Simpson's bias. However, it is not. It can be seen from Table~\ref{tab:hyper-para} that when the learning rate is fixed, and only the batch size is changed, the larger the value of $\epsilon$, the smaller the end-metric (F1) on the training set. Meanwhile, if we fix the batch size and observe the effect of the learning rate on Simpson's bias, the same conclusion can be obtained. This relationship exists equally for the model quality on dev dataset $\mathrm{F1}^\texttt{Dice}_\texttt{dev}$ and the significance of bias $\epsilon$ as well, as can be see from Figure~\ref{fig:3-2}. The performance does not decrease with the increase of the $\epsilon$ when $\epsilon$ at a lower value. Together all the figures in Figure~\ref{fig: bias-vs-quality} we think it is mainly due to the model generalization

It may suggest that the training batch size has no special effect on $\epsilon$, possibly because of the training batch size $\ll$ the corpus size. Moreover, changing the model type does not affect the correlation here. These experimental results suggest that the Simpson's bias is a common phenomenon in NLP training and not changing with model tuning.
\fi

\end{document}